\documentclass[lettersize,journal]{IEEEtran}

\usepackage[english]{babel}
\usepackage[utf8]{inputenc}
\usepackage{times} 
\usepackage{cite}

\usepackage{hyperref}
\usepackage[nolist]{acronym}

\usepackage{multirow}
\usepackage{makecell}

\usepackage{graphicx} 
\usepackage{epsfig} 
\usepackage{pgfplots}
\usepackage{caption}
\usepackage[font=footnotesize]{subcaption}
\usepackage{pgfplots}

\usepackage{float}

\usepackage[export]{adjustbox}

\newcommand{\maxs}{\mathit{max}}
\newcommand{\mins}{\mathit{min}}
\usepackage{mathtools}
\usepackage{nicefrac}
\usepackage{amsmath}
\usepackage{amsthm}
\usepackage{amssymb}
\usepackage{bm} 
\usepackage{balance}

\DeclareMathOperator*{\minimize}{minimize}

\DeclareMathOperator*{\argmax}{arg\,max}

\newtheorem{theorem}{Theorem}

\usepackage{algorithm}
\usepackage[noend]{algpseudocode}

\usepackage{pgfgantt}
\usepackage{dtsc-creafig}

\usepackage{tikz}
\pgfdeclarelayer{foreground}
\pgfsetlayers{background, main, foreground}
\usetikzlibrary{quotes, angles, backgrounds, arrows, automata, shapes, positioning, calc, through, spy, decorations.pathreplacing, decorations.markings, arrows.meta, automata, petri}

\tikzset{
    imglabel/.style={
      rectangle,
      inner sep=2pt,
      text=black,
      minimum height=1em,
      text centered,
      fill=white,
      fill opacity=1.0,
      text opacity=1,
      anchor=south west,
    },
  }
\tikzset{
	state/.style={
		rectangle,
		draw=black, very thick,
		minimum height=1.0em,
		text centered,
	},
}
\tikzset{
  on each segment/.style={
    decorate,
    decoration={
      show path construction,
      moveto code={},
      lineto code={
        \path [#1]
        (\tikzinputsegmentfirst) -- (\tikzinputsegmentlast);
      },
      curveto code={
        \path [#1] (\tikzinputsegmentfirst)
        .. controls
        (\tikzinputsegmentsupporta) and (\tikzinputsegmentsupportb)
        ..
        (\tikzinputsegmentlast);
      },
      closepath code={
        \path [#1]
        (\tikzinputsegmentfirst) -- (\tikzinputsegmentlast);
      },
    },
  },
  mid arrow/.style={postaction={decorate,decoration={
        markings,
        mark=at position .5 with {\arrow[#1]{stealth}}
      }}},
}

\usepackage{csvsimple}
\usepackage{booktabs}
\usepackage{adjustbox}

\graphicspath{{./figures/}}

\begin{document}

\title{Heterogeneous Multi-robot Task Allocation for Long-Endurance Missions in Dynamic Scenarios} 

\author{{\'{A}lvaro} Calvo and Jes{\'{u}}s Capit{\'{a}}n  
    %
    \thanks{This work was supported by project PID2024-161069OB-C33, funded by MICIU/AEI/10.13039/501100011033 and FEDER, UE.}
    \thanks{The authors are with the Multi-robot \& Control Systems group, Universidad de Sevilla, Spain, {\tt\small \{acalvom,jcapitan\}@us.es.} }
}

\maketitle


\begin{acronym}
  \acro{ARP}[ARP]{Arc Routing Problem}
  \acro{BT}[BT]{Behavior Tree}
  \acro{CARP}[CARP]{Capacitated Arc Routing Problem}
  \acro{LKH}[LKH]{Lin-Kernighan-Helsgaun}
  \acro{MILP}[MILP]{Mixed-Integer Linear Program}
  \acro{MRTA}[MRTA]{Multi-Robot Task Allocation}
  \acro{TSP}[TSP]{Traveling Salesman Problem}
  \acro{UAV}[UAV]{Unmanned Aerial Vehicle}
  \acro{VNS}[VNS]{Variable Neighborhood Search}
  \acro{VRP}[VRP]{Vehicle Routing Problem}
  \acro{wrt}[w.r.t.]{with respect to}
  \acro{ILS}{Iterated Local Search}
  \acro{SA}{Simulated Annealing}
\end{acronym}


\begin{abstract}
   
  We present a framework for \ac{MRTA} in heterogeneous teams performing long-endurance missions in dynamic scenarios. Given the limited battery of robots, especially for aerial vehicles, we allow for robot recharges and the possibility of fragmenting and/or relaying certain tasks. We also address tasks that must be performed by a coalition of robots in a coordinated manner. Given these features, we introduce a new class of heterogeneous \ac{MRTA} problems which we analyze theoretically and optimally formulate as a \ac{MILP}. We then contribute a heuristic algorithm to compute approximate solutions and integrate it into a mission planning and execution architecture capable of reacting to unexpected events by repairing or recomputing plans online. Our experimental results show the relevance of our newly formulated problem in a realistic use case for inspection with aerial robots. We assess the performance of our heuristic solver in comparison with other variants and with exact optimal solutions in small-scale scenarios. In addition, we evaluate the ability of our replanning framework to repair plans online.

\end{abstract}





\section{Introduction}
\label{sec:introduction}



\IEEEPARstart{T}{he} use of heterogeneous robot teams is rapidly expanding in applications that benefit from the combination of different robot capabilities, such as inspection~\cite{calvo_icuas22,agarwal_tro24}, agriculture~\cite{Ferreira2024,Leahy2022}, or fire fighting~\cite{seraj_tro22}. For example, \acp{UAV} could be combined with ground robots~\cite{Ferreira2024,krizmancic_ral20}, as the former can access more distant places while the latter can carry heavier equipment. Cooperation among heterogeneous \acp{UAV} is another option~\cite{smith_auro19,seraj_tro22,calvo_icuas22}, as they may provide different maneuverability (e.g., rotary vs. fixed-wing vehicles) or sensors and manipulation/delivery capabilities depending on the vehicle. In these applications, \ac{MRTA} can become especially hard, as they usually pose a multi-objective optimization with multiple constraints; e.g., some tasks may only be executed by certain robots with the required capabilities, or some could need to be accomplished by multiple robots in a synchronous manner, among others. 
We are interested in long-endurance missions in outdoor environments. This setting brings two additional challenges: 1) battery capacities could be limiting for the robots, especially for multi-\ac{UAV} teams, which inspire our work, so recharging operations should be scheduled during operation; and 2) these outdoor scenarios are typically dynamic and require online replanning to perform long-endurance missions robustly, potentially dealing with robot delays and failures. 

In this paper, our objective is to devise a planning framework to solve a new class of heterogeneous \ac{MRTA} problems. We introduce recharge operations to extend robots' autonomy, and we endow the problem with greater flexibility by allowing certain tasks to be fragmented and/or executed by coalitions -- with a non-fixed size -- of robots, as well as allowing the possibility of having inter-robot relays. Moreover, as we aim for dynamic scenarios where unplanned events may occur, our framework is able to detect these events and repair or fully recompute plans in real time to adapt to the new circumstances. To the best of our knowledge, there is no alternative method in the literature that combines all these problem features (see related works in Section~\ref{sec:relatedWork}). Therefore, we take a step forward by categorizing, formulating, and solving this novel class of heterogeneous \ac{MRTA} problems for long-endurance missions.  
The contributions of this work are the following:

\begin{enumerate}
  \item Motivated by heterogeneous teams of robots performing long-endurance missions in dynamic outdoor settings, we pose a novel \ac{MRTA} problem (Section~\ref{sec:problem_description}) where robot recharges are allowed and different types of tasks are combined, depending on their level of decomposability and on the size flexibility of the robot coalition required to execute them. We discuss the categorization of our problem within well-known \ac{MRTA} taxonomies in the literature and prove its NP-hard complexity.
  \item We formulate our \ac{MRTA} problem as a \ac{MILP}. Our formulation is general enough to account for all problem features (Section~\ref{sec:milp}), integrating  heterogeneous robot capabilities, recharge operations, task decomposition, inter-robot relays, and multi-robot tasks performed by synchronized coalitions of fixed or variable size. From a theoretical perspective, this \ac{MILP} formulation helps analyze the complexity and characteristics of the optimal solutions of our new class of \ac{MRTA} problems. 
  \item Given the NP-hardness of the problem, exact solvers suffer from scalability issues to tackle the  \ac{MILP} posed here. Therefore, we contribute a novel heuristic algorithm that leverages specific properties of the problem in order to find approximate solutions that comply efficiently with all constraints (Section~\ref{sec:planner}). 
  \item To robustly cope with dynamic scenarios, we integrate our heuristic solver into a mission planning and execution framework (Section~\ref{sec:online_replaning}). Mission execution is monitored to 1) repair plans in case of robot delays; and 2) recompute the full plan online if a repair is not possible or an unexpected event occurs, such as a robot failure or the arrival of new tasks. 
  \item We provide extensive experimental results to demonstrate our algorithms in a realistic use case for multi-\ac{UAV} inspection (Section~\ref{sec:experimental_results}). We compare our heuristic plans with the optimal solutions of the \ac{MILP} formulation in small-scale scenarios. We then assess the performance of the heuristic solver for larger scenarios and evaluate our whole replanning and execution framework. All our code is provided open-source for the community and integrated into a ROS-based architecture.\footnote{Robot Operating System, \url{https://www.ros.org}.}     

\end{enumerate}

This paper is an \emph{evolved version} of our previous work~\cite{calvo_icra24}, where we introduced our \ac{MILP} formulation. We extend it with the heuristic algorithm to find approximate solutions efficiently and with the online replanning framework for dynamic scenarios. Additionally, we provide a full set of new experimental results and release all the code as open source.



\section{Related Work}
\label{sec:relatedWork}

A number of comprehensive~\ac{MRTA} taxonomies can be found in the literature~\cite{gerkey2004formal,nunes2017taxonomy,korsah2013comprehensive}. The existing methods can be classed as ``centralized,'' where a single entity has access to all robots' information to perform task allocation, or ``distributed,'' in which  robots compute their plans locally and exchange information with others to converge to a common solution. They can also be grouped into exact methods, which provide optimal (or near-optimal) solutions, and methods that can provide approximate solutions more efficiently (typically using heuristics) and are more suitable for online task allocation.

\subsection{Distributed approaches}

Market-based approaches have been widely used for \ac{MRTA}, such as auctions~\cite{choi_tro09,otte_auro20}, which involve auctioning tasks to robots through a bidding process based on a utility function that combines the robot’s capabilities and the problem constraints. Even though they can be centralized, these methods are typically implemented in a distributed fashion, where each robot determines its own bid for tasks through an internal cost function. 
Besides auctions, other distributed greedy methods, such as distributed versions of the Hungarian algorithm~\cite{chopra_tro17} and task-swapping algorithms~\cite{liu_auro15}, iterate through pairs of robots, exchanging tasks to improve team performance.

Auction-based methods rarely consider schedules and task ordering constraints. The work by Krizmancic et al.~\cite{krizmancic_ral20} is an exception, as they address temporal and precedence constraints for task allocation in aerial-ground robot teams for automated construction. Others have also proposed decentralized auctions for multi-robot task scheduling with time windows, considering either task precedences~\cite{mcintire_aamas16} or robot capacities~\cite{bai_tase23}.

Ferreira et al.~\cite{Ferreira2024} present a distributed algorithm for task scheduling, considering precedence constraints but not multi-robot tasks. They propose an~\ac{MILP} formulation (and a genetic solver) in which robots have heterogeneous capabilities and battery constraints, but do not model recharges. Another approach is to use probabilistic methods, as Smith et al. do~\cite{smith_auro19}, with a distributed algorithm based on a Monte Carlo tree search. They consider battery-limited robots but do not address temporal constraints or multi-robot tasks. Generally, these distributed methods struggle to find optimal solutions in complex problems with tightly coupled restrictions, such as those targeted in this work.

\subsection{Multi-agent planning and scheduling}

The multi-agent planning community has devoted significant effort to solving task planning and scheduling problems~\cite{torreno_cs17}. In these scenarios, each robot's plan consists of an ordered set of actions (or tasks) that can have varying durations and can be executed sequentially or in parallel, with strong time-related dependencies. Some works address uncertainty in task durations or outcomes~\cite{dhanaraj_icra24,choudhury_auro22}, but few focus on multi-robot tasks where a coalition performs the same task concurrently. 
For instance, by assuming minimum and maximum coalition sizes, multi-robot tasks can be decomposed into single-robot tasks solvable by classical single-agent planning, with solutions subsequently combined~\cite{shekhar_ai20}. A separate collaborative action would be included for each possible combination of robots that can execute each multi-robot task. The limit on the number of coalition members has also been used to partition the problem into loosely and tightly coupled components and solve them separately, incrementally increasing the number of robots that share a task~\cite{chouhan_ai17}. Unlike these approaches, our method does not require fixing intervals for the number of robots in a task, which allows for greater flexibility. Additionally, in contrast to these prior approaches, we prioritize planning for recharges over task ordering constraints.

\subsection{\ac{MILP} formulations}

Operations research offers a variety of related problem formulations. \acp{ARP} involve covering a set of graph edges (arcs) with one or more vehicles~\cite{corberan_networks21}. The \ac{CARP} adds limited vehicle capacities per tour, suitable for modeling battery-limited robots. In robotics, \acp{VRP} and their variants~\cite{vidal_ejor20} are more common. \acp{VRP} involve computing tours for one or more vehicles to visit spatial locations (graph nodes), starting and returning to one or more depots. \acp{VRP} generalize the \ac{TSP}, where the vehicle must return to the initial depot. These NP-hard problems are often addressed with heuristic methods due to the limitations of \ac{MILP} formulations.

Agarwal et al.~\cite{agarwal_tro24} proposed a \ac{CARP} for power line inspection with multiple \acp{UAV}, modeling directional and state-dependent (inspecting vs. traveling) edge costs, such as wind effects. Vehicle battery life is limited per tour, but recharges are not considered. While many \ac{ARP}/\ac{VRP} formulations incorporate battery constraints or heterogeneous capabilities, recharging operations and inter-robot synchronization for multi-robot tasks are rare. For instance, Dorling et al.~\cite{Dorling2017} presented a \ac{UAV} delivery \ac{VRP} with vehicle reuse (recharging), including deadlines but not multi-robot tasks. Li et al.~\cite{li_cie18} proposed a multi-period \ac{MILP} formulation (without time constraints) to model recharges in a multi-\ac{UAV} traffic monitoring \ac{CARP} variant.
  
Planning recharge times and locations for \emph{persistent} \ac{UAV} teams has garnered significant interest due to its relevance to aerial monitoring and surveillance. Several works~\cite{Mathew2015,ding_iros19,yu_icra18,diller_aamas23,Maini2019} have presented \ac{TSP} variations that provide schedules specifying when and where \acp{UAV} recharge between tours. Mathew et al.~\cite{Mathew2015} consider multiple moving ground recharging stations, discretizing \ac{UAV} trajectories into projected ground points for a graph-based abstraction. Ding et al.~\cite{ding_iros19} also consider mobile depots, formulating a generalized multiple depot \ac{TSP} for persistent \ac{UAV} surveillance. Yu et al.~\cite{yu_icra18} plan recharge-inclusive tours for a single \ac{UAV}, optimizing visit order, recharge times, and locations. Diller et al.~\cite{diller_aamas23} propose a mixed-integer nonlinear program for joint path planning of a \ac{UAV} and a moving ground vehicle, allowing the \ac{UAV} to adjust its speed between rendezvous for recharging. Maini et al.~\cite{Maini2019} minimize coverage time by optimizing routes and rendezvous locations for a \ac{UAV}-ground vehicle team. Arribas et al.~\cite{arribas_tro23} address persistent aerial service, where each location requires continuous \ac{UAV} coverage, minimizing fleet size while ensuring persistent service with recharges. While these works focus on persistent tasks, our work addresses a broader problem encompassing fragmentable and relayable multi-robot heterogeneous tasks.

Another important aspect of the problem addressed in this work is the existence of temporal constraints and multi-robot tasks; i.e., tasks that have to be executed by several robots working in a synchronized fashion. A first step to tackle this is forming coalitions made up of robots with heterogeneous capabilities, which collectively fulfil the task requirements. For instance, Ramchurn et al.~\cite{ramchurn_aamas10} propose an \ac{MILP} formulation where coalitions consist of robots with the same capabilities but different degrees of effectiveness performing multi-robot tasks with time deadlines. Gosrich et al.~\cite{Walker2023} developed a mixed integer non-linear programming approach to formulate a \ac{MRTA} problem with task precedence constraints and robot coalitions for multi-robot tasks. A second step is to deal with the temporal constraints. For this, we focus on task deadlines and robot synchronization for multi-robot tasks and relayable tasks. Bredstrom et al.~\cite{bredstrom_ejor08} proposed a \ac{VRP} with time windows where some visits must be pairwise synchronized due to application requirements. Instead of minimizing waiting times as we do, they formulate hard constraints that impose an equal arrival time on certain robots performing different tasks. Flushing et al.~\cite{flushing_iros14} presented an \ac{MILP} for heterogeneous \ac{MRTA}, where \emph{non-atomic} tasks are considered; i.e., tasks that can be executed incrementally over disjoint periods of time. This is related to our concept of fragmentable tasks, which are tasks split into several segments, though we also add the possibility of having relayable tasks, which implies additional synchronization restrictions between the robot leaving the task and the one picking it up. These works do not typically consider reusable robots with limited operation time that can be reused by means of periodic recharges. Overall, among all existing \ac{MILP} formulations in the state of the art, we believe that there is a gap consisting of methods combining all the features addressed in this work. Either multi-robot, fragmentable, and relayable tasks are not considered, or vehicle recharges are not modeled. 

\subsection{Heuristics}

Due to the high complexity of the multi-robot optimization problems discussed in this work, to cope with either heterogeneous robot capabilities or multi-robot and fragmentable tasks with temporal constraints, finding optimal solutions is usually computationally demanding, even prohibitive in certain cases. Therefore, heuristic algorithms are commonly proposed to compute approximate solutions in a more efficient and scalable manner. 

Metaheuristic algorithms are a first approach; tabu search, genetic algorithms, and simulated annealing being the most common options for \ac{MRTA} problems. For instance, genetic algorithms have been applied successfully to solve problems with heterogeneous robots and tasks with temporal and precedence constraints~\cite{Ferreira2024,miloradovic_tcyber22,bai_is18,yan_esa24}. Simulated annealing has been used to solve delivery problems where \acp{UAV} can make multiple trips by recharging~\cite{Dorling2017}, and recently, to solve scheduling problems for multi-robot manufacturing~\cite{Liu2023}. The metaheuristic \ac{VNS} has also been used for orienteering problems in multi-\ac{UAV} data collection applications~\cite{PenikaVNS2019,PenikaVPN2019}.

On the one hand, metaheuristics have the advantage that they are generic algorithms that can be adapted to encode many diverse problems. On the other hand, as they rely on random search, they can require a large amount of  computation time to identify promising solutions. Another option is to use heuristics specifically designed for the application at hand, trying to leverage any a priori known information about the problem structure.
In this category, it is worth noting some heuristics for well-known problems in the literature that are related to ours. For instance, the \ac{LKH} heuristic~\cite{Helsgaun2017}, which follows a strategy based on local search, has been widely used for different variants of the \ac{TSP}, including problems that consider \ac{UAV} recharges~\cite{Maini2019,Mathew2015}. Agarwal et al.~\cite{agarwal_tro24} proposed a novel heuristic for a multi-\ac{UAV} \ac{CARP} problem. The underlying idea is to create an initial set of tours (one per arc to visit) which are then merged together in a greedy fashion to improve the solution. Al-Hussaini et al.~\cite{Gupta2023} discussed existing heuristics for scheduling problems, and proposed a new one for a problem with multi-robot tasks and temporal constraints. Gombolay et al.~\cite{Gombolay2018} presented another heuristic algorithm supporting temporal windows and precedence constraints for the tasks, but also heterogeneous robot capabilities. They use a multi-agent task sequencer, inspired by real-time processor scheduling techniques, in conjunction with an~\ac{MILP} solver that resolves task-robot allocation. Although they outperform other relevant state-of-the-art heuristic solvers, yielding nearly-optimal solutions, battery-limited robots and multi-robot tasks are not modelled. Messing et al.~\cite{Messing2022} contributed a unified framework for task planning and scheduling in heterogeneous multi-robot teams. Their solution is a multi-layer approach that interleaves task planning, allocation, scheduling, and motion planning until a valid plan is found, applying heuristic search-based algorithms at the different layers. The framework is quite generic and includes task temporal constraints and multi-robot tasks performed by robot coalitions (they add wait constraints to synchronize those coalitions). However, they do not cover battery constraints and recharging operations. Ramchurn et al.~\cite{ramchurn_aamas10} did not cover battery constraints either, but they devised an interesting heuristic procedure to solve a problem where coalitions of heterogeneous robots have to be allocated to multi-robot tasks with deadlines. 

Finally, reinforcement learning is a promising alternative to heuristic methods for addressing the computational burden of \ac{MRTA}~\cite{ma_tiv24,limbu_icra24}. However, key challenges with these methods include their lower interpretability and lack of optimality guarantees.

\subsection{Replanning approaches}

Apart from the traditional methods for \emph{offline} task allocation, there are specific frameworks in the literature for dynamic \ac{MRTA}, which monitor task execution during operation and trigger some replanning procedure if the running plan is not valid anymore after an unexpected event has occurred (e.g., a change in some task requirement or robot capabilities). The idea is to reallocate incomplete tasks not from scratch but taking into account the previous plan and the effects of the disruptive event on task performance.    

Al-Hussaini et al.~\cite{Gupta2023} presented a heuristic algorithm that provides reallocation suggestions to handle contingencies during mission execution, such as a robot failing, a new task arriving, a task changing its parameters, a task being identified as risky, or other events of this nature. They consider multi-robot tasks with deadlines and precedence constraints as well as the possibility of creating additional rescue and relay tasks, the former to rescue a disabled robot, the latter to send a robot close enough to a subgroup of teammates out of communication range in order to share a new plan with them. Neville et al.~\cite{Neville2023} also proposed an approach for dynamic task allocation in heterogeneous multi-robot settings with task temporal constraints. Multi-robot tasks are performed by robot coalitions that are created using a trait-based method. Overall, their method interleaves task allocation, scheduling, and motion planning in order to compute solutions incrementally by performing graph-based heuristic searches. After an unexpected event, a new solution is efficiently computed by repairing only the necessary nodes in the task graph. Leahy et al.~\cite{Leahy2022} developed a framework for task planning in teams of robots with heterogeneous capabilities. They define a specification language based on temporal logic, which a user can employ for high-level mission specification, allowing temporal constraints and multi-robot tasks. These specifications are then encoded in an \ac{MILP} formulation whose objective is optimizing robustness against robot failures. Thus, the plan computed is the one that tolerates the largest number of robot dropouts. Online replanning is enabled by means of another \ac{MILP} that modifies the original after a robot dropout. Recently, \cite{Lippi2023} presented an \ac{MILP} formulation for task allocation in human multi-robot teaming scenarios. After the mission has started, execution of the plan is continuously monitored, and a replanning procedure is triggered if the current plan is no longer feasible or if its expected quality (according to the application cost function) has decreased below a certain threshold. Unlike ours, these methods do not consider battery-limited robots with the possibility of recharging, although the concept of creating \emph{virtual} robot rescue tasks~\cite{Gupta2023} may be similar to our recharge tasks.



\section{Problem Description}
\label{sec:problem_description}

We address a task planning problem with a team of heterogeneous cooperating robots. The robots are heterogeneous in the sense that they can provide different sensing/locomotion capabilities, and not all of them are suitable for every task. For instance, a robot with the ability to manipulate and transport items will be suitable for a delivery task, while a robot with a specific camera onboard could be required for a particular inspection task. Each robot has a limited operational time given by its battery capacity, but unlimited recharges are allowed (with an associated time cost) at fixed stations with known positions to reset the robot's battery level. Each task has a known spatial location and a predefined time duration; we do not have specific time windows in which tasks need to start their execution, nor precedence constraints, but only a maximum completion time for each task (\emph{deadline}), as a way to establish different priorities between tasks. 

The goal is to compute the optimal plan (minimizing the \emph{makespan}; i.e., the completion time of the whole mission) for the team, given all constraints. This means devising a schedule containing the set of ordered tasks that each robot has to execute and their start time instants. Moreover, we consider dynamic scenarios, in the sense that new tasks may arrive at any time and robots may fail when executing an assigned task or while traveling. This implies the need for online replanning in order to react to new tasks or circumstances.

One of the major novelties of this work is given by the complexity of the allocation problem that we tackle, as we consider different types of tasks that yield a new task categorization. Before describing these task categories, let us propose a running example that will be used to motivate our work and to better illustrate the different types of tasks. 

\paragraph*{\emph{\textbf{Running example}}} \textit{Our main focus in this work is the use of heterogeneous teams of \acp{UAV} to execute long-endurance missions in outdoor settings. In these scenarios, the limited flight time of \acp{UAV} is key, and this is why we explicitly allow recharging operations when planning, to permit extended periods of autonomy. 
Thus, imagine an application where a team of \acp{UAV} provides support to human workers during inspection operations in a solar energy plant. Depending on their capabilities, the \acp{UAV} could be sent to inspect remote areas of the plant, to monitor worker operations for safety issues, to deliver tools or other items to some workers, and so on. One or several base stations would be installed at known positions around the plant so that the \acp{UAV} can recharge when needed. Although there may be a starting set of scheduled tasks, as the inspection mission evolves, human operators could decide to order new tasks to be assigned to the supporting \acp{UAV}, which makes the mission dynamic. Moreover, due to unexpected situations or hardware issues, \acp{UAV} could run out of battery and become unavailable.} 

\begin{figure}[tb]
  \centering
  \includegraphics[width=0.7\columnwidth]{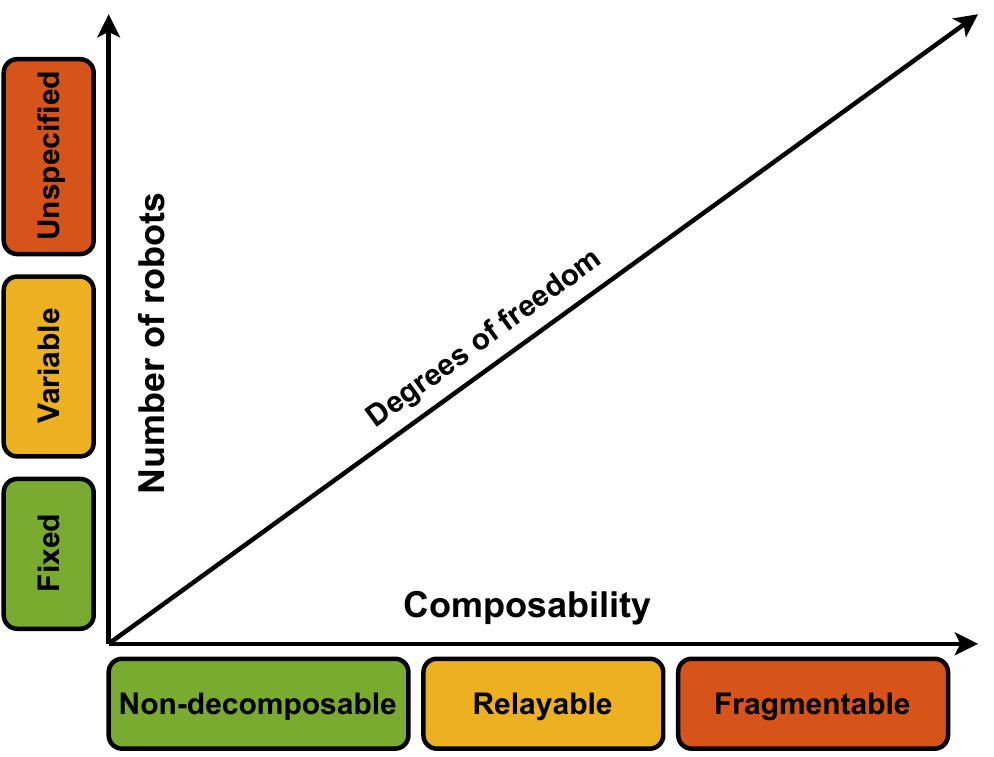}
  \caption{Complexity of the problem according to task categorization. Tasks are classified according to their decomposability and coalition size flexibility; a higher degree of freedom implies a planning problem that is harder to solve.}
  \label{fig:features_overview}
\end{figure}

Now let us describe the task categorization that we consider in our problem. As shown in Figure~\ref{fig:features_overview}, tasks are classified according to two properties that determine the overall complexity of the problem: task decomposability and task coalition size flexibility. 

\begin{itemize}
  \item \textbf{Decomposability}: This property indicates the capacity of a task to be decomposed into subtasks. We consider three types: \emph{non-decomposable}, \emph{fragmentable} and \emph{relayable}. Non-decomposable tasks have to be executed entirely by the same set of robots from start to end without pause. Fragmentable tasks can be divided into a set of fragments (each fragment duration is obtained by also dividing the original task duration) that can be executed independently by different robots. There are no dependencies between the timelines of the various fragments: they could be executed in any order, with gaps in between, and even overlapping in time. In contrast, relayable tasks can be divided into a sequence of fragments that must be executed on a continuous timeline without gaps in between. This division is made when the assigned robot (or multi-robot coalition) does not have enough battery to perform the entire task and there is a need for relays. Therefore, in relayable tasks, consecutive fragments are executed by different robots, as robots leaving to recharge are replaced by new ones. In our running example, a non-decomposable task could be a \ac{UAV} that has to deliver a tool to an operator; this task has to be carried out entirely by the same robot without interruptions. A fragmentable task could be the aerial inspection of a particular area of the solar plant to search for malfunctioning elements. In this case, the inspection area could be divided into several parts, and these could be inspected independently. For example, the same \ac{UAV} could inspect all subareas (task fragments) sequentially, even including recharges in between; or the subareas could be inspected by different \acp{UAV} in parallel. A relayable task could be a \ac{UAV} that has to use its onboard camera to monitor a worker operating in the field for safety issues. When the operation is risky, it is critical to have a continuous video stream of the worker. Therefore, if the task duration is too long, successive \acp{UAV} will have to relay each other as they run out of battery, in such a way that one is always monitoring the operation.
  
  \item \textbf{Coalition size flexibility}: We consider multi-robot tasks that need to be jointly executed by a coalition of robots. This property refers to the flexibility of the task with respect to the required size of the coalition. We consider three types of specification for the coalition size. First, for tasks where the coalition size is \emph{fixed} (it could be 1 for single-robot tasks) as a hard constraint, these tasks must be executed by coalitions with exactly the specified size. Second are tasks where the coalition size is \emph{variable} as a soft constraint. This means that an ideal coalition size for the task is specified, but coalitions of different size are also allowed although with a penalty. The third kind is tasks where the coalition size is \emph{unspecified}. These are tasks without constraints on the coalition size. Any coalition size is allowed without penalty and the task duration will depend on the final number of robots allocated. We assume that unspecified tasks are also fragmentable, which is usually the case, but not the opposite. Coming back to our running example, a task to monitor a worker's operation in the field could be requested to be executed by three \acp{UAV} simultaneously, but it may still be acceptable to execute it with fewer \acp{UAV} if there are not enough resources. This would be an example of a variable coalition size. In contrast, some inspection techniques for solar panels require the explicit cooperation of a fixed number of several \acp{UAV}. For example, photoluminescence inspection involves illuminating a specific portion of the solar panel surface and recording the luminescence emission generated in the remaining area. This could be done by a coalition of exactly two \acp{UAV}, one illuminating and the other acquiring images with a specialized camera onboard. An example of a task with an unspecified coalition size is a survey of a given area for standard thermal inspection. In this case, the more \acp{UAV} assigned, the shorter the task duration. 

\end{itemize}

We cover what are known as ST-MR-TA problems according to a well-known taxonomy of \ac{MRTA}~\cite{gerkey2004formal}. This means that our robots are \emph{Single-Task} (ST); they can only perform one task at a time. The tasks are \emph{Multi-Robot} (MR); they could require multiple robots to be executed. Lastly, we have a \emph{Time-extended Assignment}; i.e., each robot is allocated several tasks that must be executed according to a given schedule. Nunes et al.~\cite{nunes2017taxonomy} extended this taxonomy to differentiate between problems with \emph{Task Windows} (TW) and \emph{Synchronization and Precedence} (SP) constraints. Our problem falls within the TW category, since we have no constraints on the start times for the tasks but we do have deadlines. Although we do not include precedence constraints between tasks, note that time synchronization is still imposed when multiple robots need to perform a task together or execute relays. Lastly, another well-known taxonomy~\cite{korsah2013comprehensive} distinguishes between different types of problems depending on the inter-task relationship. Our decomposable tasks fall into the \emph{Complex Dependencies} (CD) category defined in this taxonomy. Given that we have multi-robot tasks, there are inter-schedule dependencies, as plans for each robot cannot be computed independently. However, CD problems imply an additional complexity, as the optimal decomposition of tasks must be computed jointly with the task allocation. This is our case, since some tasks may be split for recharges and there is an additional degree of freedom to decide when to recharge.  

Finally, note that we do not allow partial execution of the tasks. All tasks must be executed completely in order for a mission plan to be valid. Since it would need to be decided to what extent each task is covered, partial execution would significantly increase the problem complexity without a clear advantage in practical scenarios (we can still propose complete solutions and replan after partial execution caused by a robot failure).  
Furthermore, \emph{persistent} tasks are not explicitly considered either. These are tasks without a finite duration to be executed indefinitely until mission termination. For instance, monitoring a given perimeter with \acp{UAV} for security. Nonetheless, note that if we include a persistent task along with others that have deadlines, the persistent task would have the lowest priority. This means it could easily be added to the robots' plan when they become idle after completing their finite tasks.
    
\subsection{Problem complexity}
\label{subsec:proof_of_NP_hardness}

\begin{theorem}
  The heterogeneous \ac{MRTA} problem proposed in this section is NP-hard.
\end{theorem}
\begin{proof}
    Given its additional flexibility in terms of task decomposability and coalition size, the problem posed here is a generalization of other well-known \ac{MRTA} problems that are NP-hard. This can be proved by contradiction: let us assume the problem is not NP-hard. But if we can find a particular instance of our problem that is NP-hard, the overall problem must also be NP-hard, as there would not be a known algorithm to solve all its instances in polynomial time. Let us consider a specific instance of the problem where all robots have unlimited battery and capabilities to execute all tasks. Let all tasks be non-decomposable, with an arbitrarily large deadline and with a fixed coalition size of 1; i.e., all tasks are single-robot. In that case there is no need for recharges, fragmentation or multi-robot synchronization, and the problem consists of assigning the tasks to the different robots, which can start or end at different base stations. This would be equivalent to the well-known multi-depot vehicle routing problem, which is NP-hard~\cite{braekers_cie16}.
\end{proof}  



\section{MILP Formulation}
\label{sec:milp}

In the following, we develop our mathematical formulation to solve the problem in Section~\ref{sec:problem_description}, cast as an~\ac{MILP}. Table~\ref{tab:list_of_symbols} contains a list of symbols used and their descriptions.

\begin{table}
  \centering
  \caption{Table of symbols separated by category. Within the \ac{MILP} variables, decision variables are shown in bold.}
  \begin{tabular}{p{.16\columnwidth}|p{.75\columnwidth}}
      \hline
      \multicolumn{2}{c}{Robot parameters} \\
      \hline
          $\mathcal{R}$ & Set of $n$ heterogeneous robots \\
          $p_{r,0}$ & Initial position of robot $r$ \\
          $v_r$ & Traveling speed of robot $r$ \\
          $B_r^{max}$ & Maximum battery time of robot $r$ \\
          $B_{r,0}$ & Initial battery time consumed of robot $r$ \\
          $B_{r}^{min}$ & Minimum remaining safety battery time for robot $r$ \\
          
      \hline
      \multicolumn{2}{c}{Task parameters} \\
      \hline
          $\mathcal{T}$ & Set of $m$ tasks to be allocated \\
          $\tilde{\mathcal{T}}$ & Set of tasks including recharge task \\
          $\mathcal{T}_0$ & Set of $n$ auxiliary tasks to encode initial robot positions \\
          $p_t$ & Spatial location of task $t$ \\
          $T_t^{e}$ & Estimated execution time for task $t$ \\
          $T_t^{max}$ & Deadline to complete task $t$ \\
          $T^d_{r,t_1,t_2}$ & Displacement time from task $t_1$ to $t_2$ for robot $r$ \\
          $N_t$ & Coalition size specified for task $t$ \\
          $t_R$ & Recharge task \\
          $H_{r,t}$ & Capability (binary) of robot $r$ to execute task $t$\\
              
      \hline
      \multicolumn{2}{c}{Problem parameters} \\
      \hline
          $\mathcal{S}$ & Set of time slots for each robot's queue \\
          $N_f$ & Maximum number of fragments for a task \\
          $\eta_1,\eta_2,\eta_3,\eta_4 $ & Normalization constants for the objective function\\
          
      \hline
      \multicolumn{2}{c}{\ac{MILP} variables} \\
      \hline

      $\bm{x_{r,t,s}}$ & It encodes if task $t$ is allocated to slot $s$ of robot $r$ \\
      $\bm{n^f_t}$ & Number of fragments in which task $t$ is divided \\    
      $n_t$ & Number of times task $t$ appears among all robot queues \\
      $n^q_t $ & Number of robot queues where task $t$ appears\\
      $n^r_t$ & Number of robots executing task $t$ simultaneously \\
      $n^q_{r,t} $ & Variable encoding if task $t$ appears in robot $r$'s queue \\
      $\bm{T^w_{r,s}}$ & Waiting time for robot $r$ in slot $s$ \\
      $T^d_{r,s}$ & Displacement time for robot $r$ to reach task in slot $s$ \\
      $T^e_{r,s}$ & Execution time for task allocated to slot $s$ of robot $r$ \\
      $T^f_{r,s}$ & Finish time for task allocated to slot $s$ of robot $r$ \\
      $B_{r,s}$ & Battery time consumed by robot $r$ at the end of slot $s$ \\
      $y_{t,r_1,s_1,r_2,s_2}$ & It encodes if robots $r_1$ and $r_2$ need to synchronize task $t$ in slots $s_1$ and $s_2$, respectively \\
      $z_{t,r_1,s_1,r_2,s_2}$ & It encodes if robots $r_1$ and $r_2$ need to relay task $t$ in $s_1$ and $s_2$, respectively \\
      $Z$ & Makespan of the mission \\
      $V_t$ & Deviation in the specified coalition size for task $t$ \\
      $\Delta T^{max}_{r,s}$ & Delay for robot $r$ to complete task in slot $s$ \\
          
      \hline
      \multicolumn{2}{c}{Heuristic planner variables} \\
      \hline

      $Q$ & List with the queue of tasks assigned to each robot \\
      $\overline{\mathcal{T}}$ & List with all task fragments to be allocated \\
      $M_R$ & Binary matrix indicating robot pre-recharges \\
      $\mathcal{R}^c$ & Set of compatible robots with a task \\
      $\mathcal{R}_t,\,\tilde{\mathcal{R}}_t$ & Best and auxiliary robot coalitions for task $t$ \\
      $\overline{B}_t$ & Lower bound for the remaining battery time of robots compatible with task $t$ \\
      $n^c_t$ & Number of robots compatible with task $t$ \\
      $n^e_t$ & Number of excess compatible robots for task $t$ \\
      $f_t$ & Recharge pattern frequency for task $t$ \\
            
      \hline
  \end{tabular}
  \label{tab:list_of_symbols}
\end{table}
    
\subsection{Preliminary definitions}

Let $\mathcal{R}$ be a set of $n$ heterogeneous robots; each robot $r \in \mathcal{R}$ has an initial position $p_{r,0} \in \mathbb{R}^3$, a traveling speed $v_r$, a maximum battery time $B_r^{\maxs}$ (battery  autonomy), and an initial battery time consumed $B_{r,0}$. Let $\mathcal{T}$ be the set of $m$ tasks to be executed by the team; each task $t \in \mathcal{T}$ has a spatial location $p_t$, an estimated execution time $T_t^{e}$, a deadline to complete the task $T_t^{\maxs}$, and a number of robots needed $N_t$ (coalition size).\footnote{Note that this requirement could be hard or soft depending on the coalition size flexibility of the task.} Apart from the actual tasks, we include an additional task $t_R$ so that the robots can recharge their batteries at one of the available base stations, defining $\tilde{\mathcal{T}} = \mathcal{T} \cup t_R$. Our formulation is agnostic to the method of selecting the \emph{best} station to recharge at each point in time, and we assume a fixed execution time $T_{t_R}^{e}$ to fully recharge batteries. 

The  heterogeneous capabilities of the robots are encoded through a set of binary variables $H_{r,t}\in \lbrace 0,1\rbrace$, where $H_{r,t} = 1$ if robot $r$ has the hardware required to execute task $t$, and $0$ otherwise. For each robot--task pair, we define a displacement time $T^d_{r,t_1,t_2}$, which is the estimated time to navigate robot $r$ from the location of task $t_1$ to the location of task $t_2$. We compute this navigation time using Euclidean distances between tasks (which may be available from a topological map) and the speed of each robot $v_r$. In order to model the initial positions of the robots, $t_1 \in \tilde{\mathcal{T}} \cup \mathcal{T}_0$ and $t_2 \in \tilde{\mathcal{T}}$, where $\mathcal{T}_0$ represents a set of $n$ auxiliary fictitious tasks, each located at the initial position of each robot $p_{r,0}$. 

To model the decision variables, we need to decide which tasks are allocated to each robot and in which order. For this, we introduce the concept of time slots: each robot has a task queue made up of a series of slots of variable duration (let $\mathcal{S}$ be the set of slots for each queue), where tasks can be allocated. The binary decision variables $x_{r,t,s} \in \lbrace 0,1\rbrace$ take value $1$ if task $t$ is assigned to slot $s$ of robot $r$, and $0$ otherwise. The duration of each slot will depend on the time required by the specific task placed in that slot, so slot durations will differ among robots, depending on their task assignment. In general, the sizes of the robot schedules (number of assigned tasks) can differ; however, for implementation purposes, all robot queues have the same size $|\mathcal{S}|$, and only the necessary slots are \emph{activated} for each robot through the variables $x_{r,t,s}$. Furthermore, since decomposable tasks can be split into several fragments, we also define the number of fragments into which each task is divided $n^f_t \in \mathbb{N}$, $n^f_t \in [1,N_f]$, where $N_f$ is the maximum number of fragments into which any task can be divided.\footnote{This variable is forced to be 1 for recharges and non-decomposable tasks.} Given the maximum battery time for robots, we can compute a bound for $N_f$ by considering the worst case of the longest task and check the number of tours that would be needed (each tour implies a new fragment and a recharge).\footnote{A similar method is used to compute a valid value for $|\mathcal{S}|$.} Note that each fragment is allocated to a different robot slot and its execution time is computed as $T^e_t / n^f_t$.


\subsection{Basic constraints}

Some basic constraints must hold regarding the assignment of tasks to slots:
\begin{subequations}\label{eq:task-slot_constraints}
  \begin{align}
    \sum_{t \in \tilde{\mathcal{T}} \cup \mathcal{T}_0} x_{r,t,s} &\leq 1, \label{subeq:one_task_per_slot}\\
    x_{r,t_R,s-1} + x_{r,t_R,s} &\leq 1, \label{subeq:no_consecutive_recharges}\\
    \sum_{t \in \tilde{\mathcal{T}}} x_{r,t,s} &\leq \sum_{t \in \tilde{\mathcal{T}}} x_{r,t,s-1}, \label{subeq:continuity}\\
    \nonumber \forall \, r &\in \mathcal{R}, \: s \in \mathcal{S}.
  \end{align}    
\end{subequations}
%
\eqref{subeq:one_task_per_slot} ensures that there is no more than one task per slot;~\eqref{subeq:no_consecutive_recharges} avoids solutions with consecutive recharges for the same robot; and ~\eqref{subeq:continuity} prevents the existence of free slots between tasks in a queue; tasks should occupy the lowest possible slots in the queue and there should be empty slots after the last assigned task. Moreover, there must be hardware compatibility for the robots assigned to a task:
\begin{equation}\label{eq:hardware_capabilities}
  x_{r,t,s} \leq H_{r,t} ,\quad \forall \, r \in \mathcal{R}, \; t \in \mathcal{T}, \; s \in \mathcal{S}.  
\end{equation}
Additionally, we define some auxiliary variables for counting: $n_t \in \mathbb{N}$ counts the total number of times task $t$ appears among all robot queues, $n^q_t \in \mathbb{N}$ counts the number of queues in which task $t$ appears, and $n^r_t \in \mathbb{N}$ is the number of robots executing task $t$ simultaneously (this is the coalition size and takes value 1 for single-robot tasks). $n^q_{r,t} \in \lbrace 0,1\rbrace$ is a binary variable that takes the value $1$ if task $t$ appears in the queue of robot $r$. Formally,
\begin{subequations}\label{eq:n_vars}
  \begin{align}
    &n_t = \sum_{r \in \mathcal{R}} \sum_{s \in \mathcal{S}} x_{r,t,s} , \label{subeq:n_t}\\
    &n^q_t = \sum_{r \in \mathcal{R}} n^q_{r,t} , \label{subeq:n_qt}\\ 
    &n^r_t \leq n^q_t , \label{subeq:n_rt}\\
    &n_t^r \geq 1 , \label{subeq:n_rt_min}\\
    &n_t = n^r_t \cdot n^f_t , \label{subeq:n_t2}\\
    &\nonumber \forall t \in \mathcal{T} \! .
  \end{align}    
\end{subequations}
Note that~\eqref{subeq:n_rt_min} means that all tasks are assigned to at least one robot, which makes sense if we assume that the multi-robot team has the required hardware and size to accomplish all tasks in the scenario, and that the battery autonomy of the robots is enough to reach each task, execute it, and return to a base station.





\subsection{Time-related variables}

We define $T^d_{r,s}$ as the time required to move robot $r$ to the location of the task allocated to slot $s$, starting at the location of its previous task, assigned to slot $s-1$; $T^e_{r,s}$ as the time required to execute the task allocated to slot $s$; and $T^f_{r,s}$ as the time when the task allocated to slot $s$ finishes.
We also define $T^w_{r,s}$ as the time that robot $r$ has to wait in slot $s$ before starting its allocated task, to coordinate task execution with other robots. Recall that we consider multi-robot tasks, where time synchronization must be enforced so that all the robots involved start the task simultaneously. A similar synchronization is needed when relaying takes place in a relayable task. This is done by establishing this \emph{waiting} time for each robot before starting, so that those arriving earlier wait for the others. The value of the waiting time will depend on the arrival time of all the robots involved.  
More formally,
\begin{subequations}\label{eq:T_r_s}
  \begin{align}
    &T^d_{r,s} = \sum_{t_2 \in \tilde{\mathcal{T}}} \big( \sum_{t_1 \in \tilde{\mathcal{T}} \cup \mathcal{T}_0} T^d_{r,t_1,t_2} \cdot x_{r,t_1,s-1}  \big) \cdot x_{r,t_2,s} ,\label{subeq:Td_r_s}\\
    &T^e_{r,s} = \sum_{t \in \tilde{\mathcal{T}}} ( T^e_t / n^f_t \cdot x_{r,t,s} ) , \label{subeq:Te_r_s}\\
    &T^f_{r,s} = T^f_{r,s-1} + T^d_{r,s} + T^w_{r,s} + T^e_{r,s} , \label{subeq:Tf_r_s}\\
    &T^f_{r,0} = 0 , \label{subeq:Tf_r_0}\\
    &\nonumber \forall \, r \in \mathcal{R}, \; s \in \mathcal{S}.
  \end{align}
\end{subequations}
%
%
%
%
$B_{r,s}$ is the battery time accumulated by robot $r$ at the end of slot $s$, which is computed recursively, with an initial value $B_{r,0}$,~\eqref{subeq:flight_time_s} taking into account that the waiting and execution times during recharge tasks do not consume battery.\footnote{If slot $s$ represents a recharge, $B_{r,s}$ is not zero but the accumulated battery time upon reaching the recharge station. This value is used to verify sufficient battery capacity to reach the station and is subsequently reset to 0 at the beginning of the next slot.} 
Since robot battery autonomy is limited, \eqref{subeq:flight_time_min} constrains the battery time consumed up to any slot to be no greater than the maximum battery time available, always leaving a minimum battery time $B_r^{\mins}$ available for safety reasons. Robots should be able to go back to a recharge station with that amount of battery time.
\begin{subequations}\label{eq:flight_time}
  \begin{align}
    &B_{r,s} = B_{r,s-1} \cdot \bar{x}_{r,t_R,s-1} + T^d_{r,s} + (T^w_{r,s} + T^e_{r,s}) \cdot \bar{x}_{r,t_R,s} , \label{subeq:flight_time_s}\\
    &B_{r,s} \leq B_r^{\maxs} - B_r^{\mins} , \label{subeq:flight_time_min}\\
    &\nonumber \forall \, r \in \mathcal{R}, \; s \in \mathcal{S} ,
  \end{align}
\end{subequations}
\noindent where $\bar{x}_{r,t,s} = 1-x_{r,t,s}$.



\subsection{Time coordination}
\label{sec:timeCoordination}

\begin{figure}[tb!]
  \centering
  \includegraphics[width=.49\columnwidth]{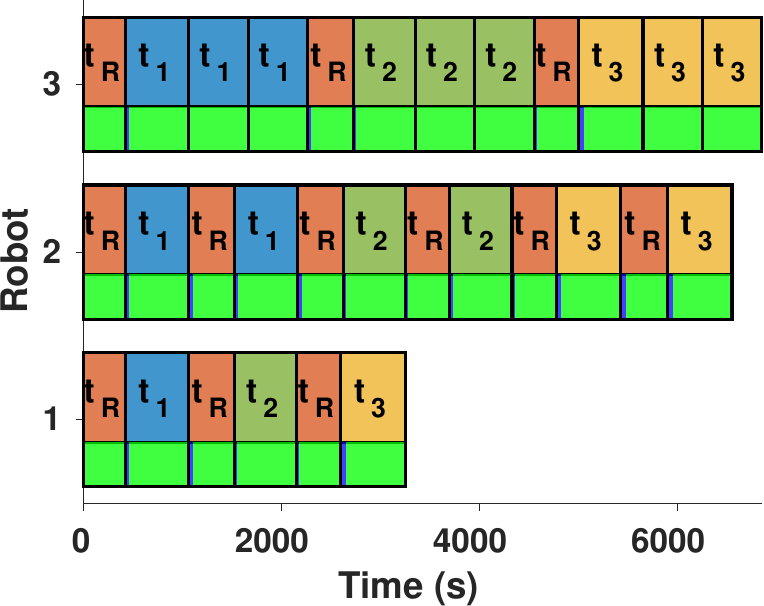}
  \includegraphics[width=.49\columnwidth]{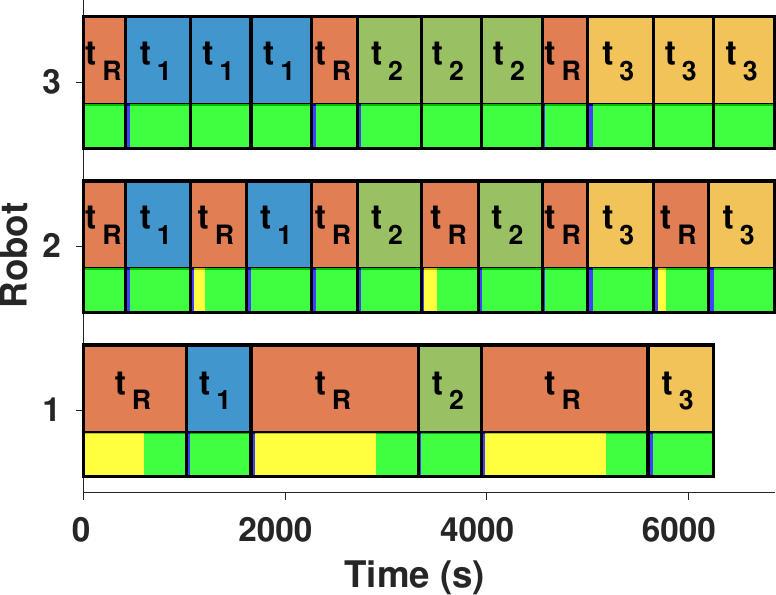}

  \caption{Example with 3 robots executing 3 consecutive multi-robot relayable tasks (with different color code). 
  For each task, displacement time is depicted in blue, waiting time in yellow and execution time in green. Each task is divided into 3 fragments ($n_t^f=3$) executed by coalitions of 2 robots ($n_t^r=2$). Robot 3 executes all tasks and recharges in between, while robots 2 and 1 relay each other to accompany robot 3. Left, solution where robots do not coordinate task execution; right, solution including time coordination constraints.}
  \label{fig:synchronization_example}
\end{figure}

Robots must coordinate their timing while executing a task in two different situations: when several robots need to perform a multi-robot task together or when a relay is carried out. Figure~\ref{fig:synchronization_example} shows an example to illustrate the difference between solutions when the schedules are coordinated and when they are not. This time coordination requires additional constraints which we model with two types of binary decision variables: $y_{t,r_1,s_1,r_2,s_2}, z_{t,r_1,s_1,r_2,s_2} \in \lbrace 0,1 \rbrace$. $y_{t,r_1,s_1,r_2,s_2}=1$ indicates that $t$ is a multi-robot task assigned to slot $s_1$ of robot $r_1$ and slot $s_2$ of robot $r_2$ and that the two robots must be synchronized; and $z_{t,r_1,s_1,r_2,s_2} = 1$ indicates that a fragment of task $t$ assigned to slot $s_1$ of robot $r_1$ is relayed by another fragment of the same task assigned to slot $s_2$ of robot $r_2$. Then robot time coordination is enforced by: 
\begin{subequations}\label{eq:time_coordination}
  \begin{align}
    & T^f_{r_1,s_1} \cdot y_{t,r_1,s_1,r_2,s_2} =  T^f_{r_2,s_2} \cdot y_{t,r_1,s_1,r_2,s_2} , \label{subeq:s_Tf} \\
    & T^f_{r'_1,s_1} \cdot z_{t,r'_1,s_1,r'_2,s_2} =  (T^f_{r'_2,s_2}-T^e_{r'_2,s_2}) \cdot z_{t,r'_1,s_1,r'_2,s_2}, \label{subeq:r_Tf}\\ 
    &\nonumber \forall \, r_1 \neq r_2,r'_1,r'_2 \in \mathcal{R}, \: s_1,s_2 \in \mathcal{S}, \: (r'_1,s_1) \neq (r'_2, s_2), \: t \in \mathcal{T} \! .
  \end{align}    
\end{subequations}
Given a pair of robots performing a multi-robot task, since the task execution time is the same for both, by setting their finish slot times to be equal through~\eqref{subeq:s_Tf}, we ensure that they start task execution simultaneously. In the case of a relay,~\eqref{subeq:r_Tf} equals the time from when the first robot finishes its task fragment to the time when the second robot starts executing its fragment (note that there may be a previous waiting time for synchronization). Although that would not be a proper relay, for implementation purposes, our formulation encodes as a \emph{virtual} self-relay when a robot performs two consecutive fragments of the same task (see an example in Figure~\ref{fig:synchronization_vars_example}). Note that~\eqref{subeq:s_Tf} holds for pairs of distinct robots synchronizing, as it does not make sense for a robot to synchronize with itself in a multi-robot task. However, in~\eqref{subeq:r_Tf} a robot could relay itself, but in that case the time slots must be different. Furthermore, for each time coordination, we need to ensure that the two slots being coordinated have the same task associated with them:
\begin{subequations}\label{eq:time_coordination2}
  \begin{align}
    & y_{t,r_1,s_1,r_2,s_2} \leq x_{r_1,t,s_1} , \label{subeq:y_leq_x1}\\
    & y_{t,r_1,s_1,r_2,s_2} \leq x_{r_2,t,s_2} , \label{subeq:y_leq_x2}\\
    & z_{t,r_1,s_1,r_2,s_2} \leq x_{r_1,t,s_1} , \label{subeq:z_leq_x1}\\
    & z_{t,r_1,s_1,r_2,s_2} \leq x_{r_2,t,s_2} , \label{subeq:z_leq_x2}\\ 
    &\nonumber \forall \, r_1,r_2 \in \mathcal{R}, \; s_1,s_2 \in \mathcal{S}, \; t \in \mathcal{T} \! .
  \end{align}    
\end{subequations}
Additionally, there are constraints on the flow of time coordination variables. First, each task fragment cannot be relayed by more than one subsequent fragment and similarly, it cannot be relaying more than one previous fragment:  
\begin{subequations}\label{eq:max_relays}
  \begin{align}
    & \sum_{r_2 \in \mathcal{R}} \sum_{s_2 \in \mathcal{S}} z_{t,r_1,s_1,r_2,s_2} \leq 1, \; \forall \, r_1 \in \mathcal{R}, \: s_1 \in \mathcal{S}, t \in \mathcal{T} \! ; \label{subeq:y_relayed_once}\\
    & \sum_{r_1 \in \mathcal{R}} \sum_{s_1 \in \mathcal{S}} z_{t,r_1,s_1,r_2,s_2} \leq 1, \; \forall \, r_2 \in \mathcal{R}, \: s_2 \in \mathcal{S}, t \in \mathcal{T} \! .\label{subeq:y_relay_once}
  \end{align}    
\end{subequations}
Note that the above limitation does not apply to variables of type $y$, since a task instance in a multi-robot task must be synchronized with all the instances corresponding to the other $n_t^r-1$ robots performing the task in parallel. Thus,~\eqref{eq:max_synchronizations} ensures that all instances of a given multi-robot task $t$ (which may also be decomposable) are grouped into sets of exactly $n^r_t$ fragments, and time synchronization only occurs between fragments belonging to the same group.
\begin{subequations}\label{eq:max_synchronizations}
  \begin{align}
     y_{t,r_1,s_1,r_2,s_2} &= y_{t,r_2,s_2,r_1,s_1} , \\ 
    \nonumber &\forall \, r_1,r_2 \in \mathcal{R}, \: s_1,s_2 \in \mathcal{S}, \: t \in \mathcal{T}, \\
     (n_t^r - 1) \cdot x_{r_1,t,s_1} &= \sum_{r_2 \in \mathcal{R}} \sum_{s_2 \in \mathcal{S}} y_{t,r_1,s_1,r_2,s_2} , \label{subeq:max_synchronizations}\\ 
    \nonumber &\forall \, r_1 \in \mathcal{R}, \:s_1 \in \mathcal{S}, \:t \in \mathcal{T} \! .
  \end{align}    
\end{subequations}
Lastly, the total number of relays for a given task is also bounded:
\begin{equation}
  n_t - n_t^r = \sum_{r_1 \in \mathcal{R}} \sum_{s_1 \in \mathcal{S}} \sum_{r_2 \in \mathcal{R}} \sum_{s_2 \in \mathcal{S}} z_{t,r_1,s_1,r_2,s_2}, \: \forall \, t \in \mathcal{T} \! . \label{eq:max_total_relays}  
\end{equation}
%
Figure~\ref{fig:synchronization_vars_example} depicts an example with its corresponding time coordination variables to illustrate the flow constraints in~\eqref{eq:max_synchronizations} and~\eqref{eq:max_total_relays}. It can be seen that all the task instances ($n_t=12$) are grouped into 3 sets of 4 fragments and, within each group, each robot synchronizes with the other 3 performing the task. Moreover, all synchronization variables $y$ are bidirectional and the total number of activated relay variables $z$ in the example is $n_t-n_t^r = 8$.

\begin{figure}[tb]
  \centering
  \includegraphics[width=0.68\columnwidth]{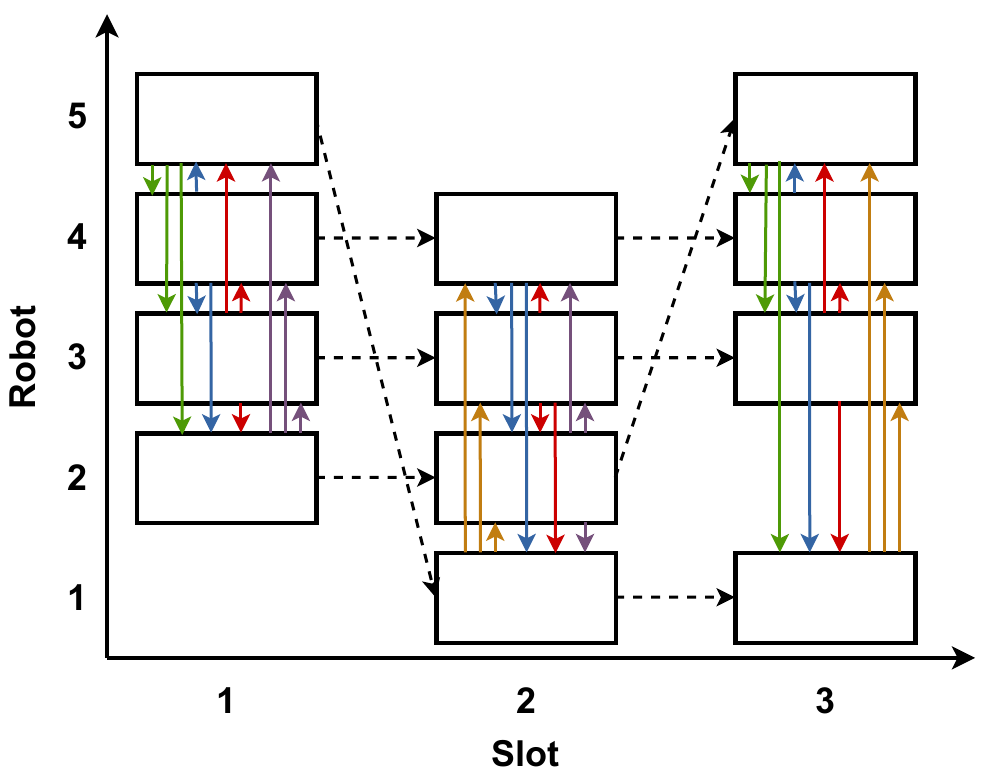}
  \vspace{-1mm}
  \caption{Time coordination example with 5 robots performing a multi-robot relayable task, which is divided into 3 fragments ($n_t^f=3$) executed by coalitions of 4 robots ($n_t^r=4$). Recharge tasks are not shown. Dashed black arrows indicate that the corresponding relay variable ($z$) is activated, and solid colored arrows indicate the activated synchronization variables ($y$), with a different color for each robot. A robot executing two consecutive fragments of the same task is modeled as a \emph{self} relay.}
  \label{fig:synchronization_vars_example}
  \vspace{-1em}
\end{figure}


\subsection{Objective function}

We propose a multi-objective cost function~\eqref{eq:costFunction} with four terms to be minimized:\footnote{Note that our formulation could easily accommodate other typical objectives such as minimizing the total traveled time by all robots.} 1) makespan, the time by which the last robot finishes its last task; 2) delays for task completion with respect to their deadlines; 3) waiting times for synchronization; and 4) deviation of the coalition size in multi-robot tasks with respect to the ideal size.
\begin{subequations}\label{eq:costFunction}
  \begin{align}
    f_{1} &= \frac{Z}{\eta_1} , \label{subeq:f1}\\
    f_{2} &= \frac{\sum_{r \in \mathcal{R}, s \in \mathcal{S}} \Delta T^{max}_{r,s}}{\eta_2} , \label{subeq:f2}\\
    f_{3} &= \frac{\sum_{r \in \mathcal{R}, s \in \mathcal{S}} T^w_{r,s}}{\eta_3} , \label{subeq:f3}\\
    f_{4} &= \frac{\sum_{t \in \mathcal{T}} V_t}{\eta_4}\label{subeq:f4} ,
  \end{align}
\end{subequations}
\noindent where $\eta_1$, $\eta_2$, $\eta_3$, and $\eta_4$ are normalization constants so that all cost terms are on the same scale.~\eqref{subeq:f1} introduces an auxiliary variable $Z$ to encode the makespan, which can be done by adding the constraint
\begin{equation}\label{eq:makespan}
  Z \geq T^f_{r,|\mathcal{S}|} , \; \forall \, r \in \mathcal{R} ,
\end{equation}
\noindent which forces $Z$ to be greater than the completion time for each robot's queue~\eqref{subeq:Tf_r_s}. $\eta_1$ is computed by considering a worst-case scenario in which a single robot executes all tasks, recharging when needed. In~\eqref{subeq:f2}, $\Delta T^{\maxs}_{r,s}$ are slack variables that represent the delay of robot $r$ in completing the task assigned to slot $s$. We set:
\begin{equation}\label{eq:d_t_max_r_s}
  \begin{split}
    \Delta T^{\maxs}_{r,s} &\geq \sum_{t \in \mathcal{T}} x_{r,t,s} \cdot ( T^{f}_{r,s} - T^{\maxs}_{t}) , \\
    \Delta T^{\maxs}_{r,s} &\geq 0 , \quad \forall \, r \in \mathcal{R}, \; s \in \mathcal{S}  .
  \end{split}
\end{equation}
Note that since it does not make sense to consider deadlines for recharges or unassigned slots,~\eqref{eq:d_t_max_r_s} reduces to $\Delta T^{\maxs}_{r,s} \geq 0$ when the task assigned to slot $s$ does not belong to $\mathcal{T}$. $\eta_2$ is computed as the maximum deadline $T^{\maxs}_t,\; \forall t \, \in \mathcal{T}$.~\eqref{subeq:f3} is the normalized overall waiting time, where $\eta_3$ is calculated as the maximum battery time autonomy $\max{\lbrace B^{\maxs}_r \rbrace},\, \forall \, r \in \mathcal{R}$.
Lastly, recall that our multi-robot tasks may accept \emph{variable} coalition sizes, so~\eqref{subeq:f4} penalizes tasks allocated a number of robots that is lower than their ideal coalition size $N_t$. This deviation in the coalition size is defined as:
\begin{equation}\label{eq:V_t}
  \begin{split}
  V_t &= N_t - n_t^r , \\ 
  V_t &\geq 0 , \\
  \forall \, t &\in \mathcal{T}, \: N_t > 0.
  \end{split}
\end{equation}
%
%
For tasks with an \emph{unspecified} coalition size, this penalty does not apply; by convention, we set $N_t=0$ for these tasks, which is why \eqref{eq:V_t} only applies when this parameter is greater than zero. For tasks with a \emph{fixed} coalition size, an extra constraint is added to force $V_t$ to be zero, making $n_t^r$ and $N_t$ coincide. $\eta_4$ is computed by adding the maximum deviation for all tasks, where this maximum deviation is $N_t - 1$; i.e., allocating a single robot to the task. 


\subsection{Optimization problem}

To sum up, our complete~\ac{MILP} problem can be formulated as follows:
\begin{equation}\label{eq:MILP}
  \begin{split}
      &\minimize_{\lbrace x_{r,t,s}, T^w_{r,s}, n_t^f \rbrace}{ f_1 + f_2 + f_3 + f_4} \\
      %
      &\;\; \text{subject to} \\
      &\quad \quad \text{task-slot assignment \eqref{eq:task-slot_constraints},} \\
      &\quad \quad \text{hardware compatibility \eqref{eq:hardware_capabilities},} \\
      &\quad \quad \text{counting variables \eqref{eq:n_vars},} \\
      &\quad \quad \text{time slot variables \eqref{eq:T_r_s},} \\
      &\quad \quad \text{battery autonomy \eqref{eq:flight_time},} \\
      &\quad \quad \text{time coordination \eqref{eq:time_coordination}--\eqref{eq:max_total_relays},} \\
      &\quad \quad \text{makespan \eqref{eq:makespan},} \\
      &\quad \quad \text{task delays \eqref{eq:d_t_max_r_s},} \\
      &\quad \quad \text{coalition size deviations \eqref{eq:V_t}.} \\
      %
      %
      %
   \end{split}
\end{equation}
%


\subsection{Linearization}

Finally, it is important to remark that some of the equations contain non-linear elements in the form of decision variables that are multiplied together. In particular, these non-linearities appear in~\eqref{subeq:Td_r_s},~\eqref{subeq:Te_r_s},~\eqref{subeq:flight_time_s},~\eqref{eq:time_coordination},~\eqref{subeq:max_synchronizations}, and~\eqref{eq:d_t_max_r_s}. In order to keep the formulation as an~\ac{MILP}, we circumvent this issue with linearization techniques by using additional auxiliary variables~\cite{Sabnis2019,Ta2004}. Basically, non-linear terms are replaced by new variables and additional linear constraints that force the value of the new variable to be equal to the term that it substitutes. 

Thus, each product of two binary decision variables $b_1 \cdot b_2$ is replaced by a new variable $b' \in \lbrace 0,1\rbrace$, and the following constraints are included to force $b'$ to be equal to the original product:
\begin{equation}\label{eq:linearization_binary_binary}
   b' \leq b_1, \quad b' \leq b_2, \quad b' \geq b_1 + b_2 - 1.
\end{equation}
In the same way, each product of a binary and a real\footnote{The same linearization is used in the case of integer variables.} decision variable $b \cdot r$, where $b \in \lbrace 0,1 \rbrace$ and $r \in \mathbb{R}$, is replaced by a new variable $r' \in \mathbb{R}$, and the following constraints are included:
\begin{equation}\label{eq:linearization_binary_real}
  \begin{split}
    &b \cdot r^{\mins} \leq r', \quad r - r^{\maxs} \cdot (1 - b) \leq r', \\ 
    &b \cdot r^{\maxs} \geq r', \quad r - r^{\mins} \cdot (1 - b) \geq r',
  \end{split}
\end{equation}
\noindent where $r^{\mins}$ and $r^{\maxs}$ are the lower and upper limits, respectively, of the real variable $r$.
For the sake of brevity, the complete set of equations after linearization is not shown here, but can be accessed in the online version of our code (see Section~\ref{sec:experimental_results}).

\section{Heuristic Planner}
\label{sec:planner}

The problem posed is NP-hard, as proved in Section~\ref{subsec:proof_of_NP_hardness}. Therefore, solving an optimal formulation such as that presented in Section~\ref{sec:milp} becomes computationally intractable as the number of robots and tasks involved increases. In this section, we propose a heuristic solver to find approximate solutions in such a way that 1) the plans comply with all problem constraints as formulated in Section~\ref{sec:milp}; and 2) they can be computed efficiently enough to operate in real time. 

There are problem-specific heuristics for \ac{MRTA} scenarios similar to ours, in which they build an initial valid solution and then iterate over it, applying operations to improve it; e.g., merging independent tours into a single tour to save costs~\cite{agarwal_tro24}, or creating new solutions by removing random robots or tasks and rearranging them~\cite{wilde_mrs23}. Given the complexity of our constraints, such a strategy would be hard to follow in our problem, as even minor operations may yield solutions that are not valid anymore: our time coordination constraints cannot always be fixed by adding recharges and/or adjusting waiting times.
Another approach is to use metaheuristic algorithms such as genetic~\cite{miloradovic_tcyber22}, simulated annealing~\cite{Liu2023}, or~\ac{LKH}, which is a widely used heuristic to solve variants of~\ac{TSP}~\cite{Maini2019, Mathew2015}. The issue is that those approaches work more poorly in complex problems with large search space, struggling to find good solutions in reasonable time. Therefore, we propose a new heuristic algorithm leveraging properties of the problem, in which tasks are ordered following certain criteria and then assigned to robot coalitions. 

\begin{algorithm}
  \footnotesize
  \caption{{ \sc{HeuristicPlanner\,($\mathcal{R},\mathcal{T}$)}}}
  \label{alg:HeuristicPlanner}
  \begin{algorithmic}[1]
    \State $\lbrace n_t^r,n_t^f,f_t\rbrace_{t \in \mathcal{T}} \gets$ \texttt{EstimateFragments}($\mathcal{R}, \mathcal{T}$)
    \For{$r$ in $\mathcal{R}$} $Q[r] \gets \emptyset$ \Comment{Initialize task queues}
    \EndFor

    \State $\overline{\mathcal{T}} \gets \emptyset$ 
    
    \For {$t$ in $\mathcal{T}$}
      \If {$t$ is  \emph{non-decomposable} \textbf{or} \emph{relayable}}
        \State $\overline{\mathcal{T}}$.\texttt{add}($t$)
      \Else
        \State $\overline{\mathcal{T}}$.\texttt{add}(\texttt{Repelem}($t,n_t^f$))
      \EndIf
    \EndFor
    \While {$\overline{\mathcal{T}} \neq \emptyset$}
        \State $Z \gets$ \texttt{ComputeMakespan}()
        \For {$t$ in $\overline{\mathcal{T}}$} \Comment{Best coalition for each task}
            \State $\lbrace \mathcal{R}_t, M_R, \Delta Z, \Delta T^w  \rbrace \gets$ \texttt{SelRobots}($\mathcal{R}, t, Z$)
        \EndFor
        
        \State $t \gets$ \texttt{Sort}($\overline{\mathcal{T}}$).\texttt{pop}() \Comment{Allocate priority task}
        \For {$r$ in $\mathcal{R}_t$}
            \State $Q[r]$.\texttt{addTaskToQueue}($t,M_R[r,t], T^f_t$)
        \EndFor

    \EndWhile

    \State \Return $Q$

  \end{algorithmic}
\end{algorithm}

Algorithm~\ref{alg:HeuristicPlanner} summarizes our heuristic solver, which receives the set of robots $\mathcal{R}$ and tasks $\mathcal{T}$ and returns a list $Q$ with the queue of tasks assigned to each robot. First, the algorithm decides the coalition size for each task ($n_t^r$) and the number of fragments into which it will be divided ($n^f_t$). This is done calling procedure \texttt{EstimateFragments} (Algorithm~\ref{alg:EstimateFragments}), which will be explained later. Then, after initializing the robot queues, a list $\overline{\mathcal{T}}$ with all task fragments to be allocated is created (lines 3--8). For each fragmentable task $t$ in $\mathcal{T}$, $n^f_t$ equal fragments are repeated and added to $\overline{\mathcal{T}}$ (procedure \texttt{Repelem}), as each fragment will be allocated to a robot coalition independently. Non-decomposable and relayable tasks are included as a single element, as they will be allocated to a coalition as a whole. In each iteration of the main loop (lines 9--15), an element of $\overline{\mathcal{T}}$ is allocated until the list is empty. Given all task assignments so far, the current makespan $Z$ of the plan is computed, and for each remaining task $t$ in $\overline{\mathcal{T}}$, the best coalition is determined calling procedure \texttt{SelRobots} (lines 11--12). This procedure will be detailed later (Algorithm~\ref{alg:RobotSelection}) and it returns the set of robots $\mathcal{R}_t$ in the best coalition selected for the task, a matrix $M_R$ with binary flags indicating whether the robots need or do not need to include a pre-recharge to execute the task, and the increase in the makespan $\Delta Z$ and in the waiting time $\Delta T^w$ introduced when assigning the selected coalition to the task. Once the best coalitions are computed for all remaining tasks, these tasks are sorted and the one with top priority is extracted (line 13). Tasks are ordered lexicographically according to several criteria:
1) tasks that are close to their deadline $T_t^{\maxs}$, giving priority to those that if not assigned now, would exceed their deadline; 2) tasks whose selected coalition introduces less $\Delta Z$; 3) tasks whose selected coalition introduces less $\Delta T^w$; 4) tasks with higher execution time ($T^e_t/n^f_t$); 5) tasks with a higher $n^r_t/n^c_t$ proportion; i.e., their coalition size with respect to the number of compatible robots for the task ($n^c_t$); and 6) tasks with less displacement time ($T^d_t$). The task with top priority is allocated to its selected coalition using procedure \texttt{addTaskToQueue} (lines 14--15), which adds task $t$ to the queue of robot $r$, including a pre-recharge task if $M_R[r,t] == 1$ and using $T^f_t$ (calculated in \texttt{SelRobots}) as coordination time to compute the waiting time. These task allocation heuristics are based on the optimization problem formulated in Equation~\ref{eq:MILP}: the task assignment order (line 13) follows similar optimization criteria, and robot coalition selection aims to minimize mission makespan. Note that the objective function in Equation~\ref{eq:MILP} cannot be directly used for task ordering, as it requires adaptation for single-task scoring.

\begin{algorithm}
  \footnotesize
  \caption{\sc{SelRobots($\mathcal{R}, t, Z$)}}
  \label{alg:RobotSelection}
  \begin{algorithmic}[1]
    \State $\mathcal{R}^c \gets$ \texttt{GetAvailableRobots}()
    

    \For {$r$ in $\mathcal{R}^c$}
      \State $M_R[r,t] \gets 1$ if a pre-recharge is needed, $0$ otherwise
      \State $T^f_r \gets$ robot's finish time after $M_R[r,t]$ and $t$
    \EndFor

    \State $\mathcal{R}_t \gets$ \texttt{Sort}($\mathcal{R}^c$).\texttt{get}($n_t^r$) \Comment{Pick earliest $n_t^r$ robots}
    
    \State $T^f_t \gets \max_{r \in \mathcal{R}_t} T^f_r$ \Comment{Task coordination time}

    \State $changed \gets$ \texttt{True}
    \While {$changed$}

      \State $changed \gets$ \texttt{False}

      \For {$r$ in $\mathcal{R}_t$}
        \State $T^w_r \gets$ waiting time for robot $r$ according to $T^f_t$
        \State $M_R[r,t] \gets$ pre-recharge flag including $T^w_r$
        \State $T^f_r \gets$ robot's finish time considering $M_R[r,t]$
      \EndFor

      \State $\tilde{\mathcal{R}}_t \gets$ \texttt{Sort}($\mathcal{R}^c$).\texttt{get}($n_t^r$) \Comment{Update robot selection}

      \If {$\tilde{\mathcal{R}}_t \neq \mathcal{R}_t$}
        \State $\mathcal{R}_t \gets \tilde{\mathcal{R}}_t$ , $T^f_t \gets \max_{r \in \mathcal{R}_t} T^f_r$

        \State $changed \gets$ \texttt{True}
      \EndIf

    \EndWhile

    \If {$T^f_t > Z$}
      $\Delta Z \gets T^f_t - Z$
    \Else
      $\hspace{1em} \Delta Z \gets 0$
    \EndIf
    \State $\Delta T^w \gets \sum_{r \in \mathcal{R}_t} T^w_r$
    
    \State \Return $\mathcal{R}_t , M_R, \Delta Z , \Delta T^w$
  \end{algorithmic}
\end{algorithm}

Algorithm~\ref{alg:RobotSelection} receives the set of robots $\mathcal{R}$ and the makespan $Z$ produced by the current robots' task queues, and selects the best robot coalition for task $t$. The algorithm selects the compatible robots that finish their task queue earliest, ensuring that the constraints in problem~\ref{eq:MILP} are satisfied after each task allocation. First, $\mathcal{R}^c$ is the set of robots with compatible hardware\footnote{For the sake of simplicity, we assume that there are enough robots to execute the available tasks, otherwise the algorithm would return that there is no valid solution.} and enough maximum battery $B_r^{\maxs}$ to execute the task and go back to a recharging station (line 1). Second, for each compatible robot, it is determined whether a pre-recharge task has to be included or not and, according to this, the finish time of the robot queue $T^f_r$ after including task $t$ is calculated (lines 2--4). Third, the compatible robots are sorted according to their $T^f_r$ and the coalition $\mathcal{R}_t$ is built selecting the earliest $n^r_t$ robots (line 5). In order to execute the task in a coordinated manner, the coordination time $T^f_t$ is computed; i.e., the time at which the latest selected robot will finish the task (line 6).
The following procedure is then repeated iteratively (lines 8--17) until the selected coalition does not change, ensuring time coordination and battery constraints: add the required waiting time to coordinate the selected robots, recompute if any pre-recharge is needed, reorder the compatible robots with the new finish queue times, and select the earliest $n^r_t$ robots. Convergence is guaranteed because the set of compatible robots is finite. Once a pre-recharge flag is set for $M_R[r, t]$, it is not reset. In each iteration, either a pre-recharge is added for a new robot, or the coalition remains unchanged. In the worst case, the loop terminates after all compatible robots have a pre-recharge. Lastly, once the final coalition for task $t$ is selected, the increase in the makespan and the waiting time are computed, assuming that task assignment (lines 18--20).

Algorithm~\ref{alg:RobotSelection} has a variation for the special case of relayable tasks that are fragmented. In that case, all the fragments belonging to that task are allocated as a whole, taking into account all the available compatible robots, instead of just picking the first $n_t^r$ (lines 5 and 14). All compatible robots are eligible for assignment to task fragments following the specific pattern in Figure~\ref{fig:template_matrix}. This pattern satisfies battery and time coordination constraints for relayable tasks, ensuring sufficient robots are always available to perform relays while others recharge. The pattern is built with two key parameters computed in Algorithm~\ref{alg:EstimateFragments}; the number of fragments into which the task is divided $n^f_t$, and the pattern frequency $f_t$, which indicates the number of fragments of task $t$ to be assigned to the same robot before introducing a recharge operation. 

\begin{algorithm}
  \footnotesize
  \caption{{\sc {EstimateFragments($\mathcal{R}, \mathcal{T}$)}}}
  \label{alg:EstimateFragments}
  \begin{algorithmic}[1]
    \For {$t \in \mathcal{T}$}

      \If {$t$ is \emph{fixed}} \Comment{Decide robots per task}
        \State $n_t^r=N_t$
      \ElsIf {$t$ is \emph{variable} \textbf{or} \emph{unspecified}}
        \State $n_t^r=1$
      \EndIf

      \State $\mathcal{R}^c \gets$ \texttt{GetCompatibleRobots}($t$)
      \State $\overline{B}_t \gets \min_{r \in \mathcal{R}^c, t_1 \in \tilde{\mathcal{T}}}( B^{max}_r - B^{min}_r - 2 \cdot\max \, T^d_{r,t_1,t})$

      \If {$T^e_t \le \overline{B}_t$ \textbf{or} $t$ is \emph{non-decomposable}}

        \State $n_t^f \gets 1$ , $f_t \gets 0$

      \ElsIf {$t$ is \emph{fragmentable}}
        \State $n^f_t \gets \lceil T^e_t / \overline{B}_t \rceil$ , $f_t \gets 0$ 

      \Else \Comment{Relayable tasks}

        \State $n^c_t = |\mathcal{R}^c|$ , $n^e_t = n^c_t - n_t^r$ 

        \State $cf \gets \lceil n_t^r / n^e_t \rceil$

        \State $n^f_t \gets \lceil cf \cdot T^e_t / \overline{B}_t \rceil$
        \State $f_t \gets \lfloor \overline{B}_t / (T^e_t / n^f_t) \rfloor$

      \EndIf
      
    \EndFor

    \State \Return $\lbrace n_t^r,n_t^f,f_t\rbrace_{t \in \mathcal{T}}$

  \end{algorithmic}
\end{algorithm}

Algorithm~\ref{alg:EstimateFragments} first determines the coalition size for each task (lines 1--5); for simplicity, tasks with unspecified and variable coalition size are allocated to a single robot. 
Then, to estimate the number of fragments, it calculates the remaining battery capacity $\overline{B}_t$ for the compatible robot with the lowest capacity (lines 6--7), after accounting for the minimum safety margin and the maximum round-trip travel time to task $t$. Non-decomposable tasks, or tasks whose execution time is shorter than $\overline{B}_t$, are not fragmented (lines 8--9). Fragmentable tasks are split into fragments with duration shorter than $\overline{B}_t$ (lines 10--11), without the need to follow any specific pattern for recharges ($f_t=0$). For relayable tasks, $\overline{B}_t$ and the number of excess compatible robots $n^e_t$ are taken into account to compute $n^f_t$ and $f_t$ (lines 12--16)\footnote{We assume sufficient robots to execute all tasks. If no excess robots were available ($n^e_t=0$), all compatible robots would have to execute the task in parallel, precluding relays.}. Depending on the number of excess robots available for relays, each robot may execute several consecutive fragments according to the allocation pattern in Figure~\ref{fig:template_matrix}.
First, $f_t$ consecutive fragments of the task are concatenated, followed by the corresponding recharge task, until a row with $n^f_t$ slots is built. 
The same procedure is followed to create $n^c_t$ rows, but shifting the pattern of each row one slot to the left (Figure~\ref{subfig:full_matrix_template}).
At this point, there may be columns with more than $n^r_t$ task fragments. For those columns, the extra fragments are replaced by recharge tasks. To reduce robot displacements and save slots, the replaced fragments are chosen from those that are preceded or followed (in their rows) by a recharge task (Figure~\ref{subfig:corrected_matrix_template}). 
After that, recharge tasks that are at the beginning or the end of a row can be removed and consecutive recharges merged in a single slot without modifying the relative positions of the remaining fragments (Figure~\ref{subfig:reduced_matrix_template}).
Although the pattern is originally built using all compatible robots $n^c_t$, after the above steps, there may be empty rows (i.e., unused compatible robots) that would be removed. 
Lastly, the remaining rows are reordered according to their starting point (Figure~\ref{subfig:sorted_reduced_matrix_template}).

\begin{figure}[tb!]  
  \begin{minipage}[t]{\linewidth}
    \raggedleft
    \includegraphics[scale=0.27]{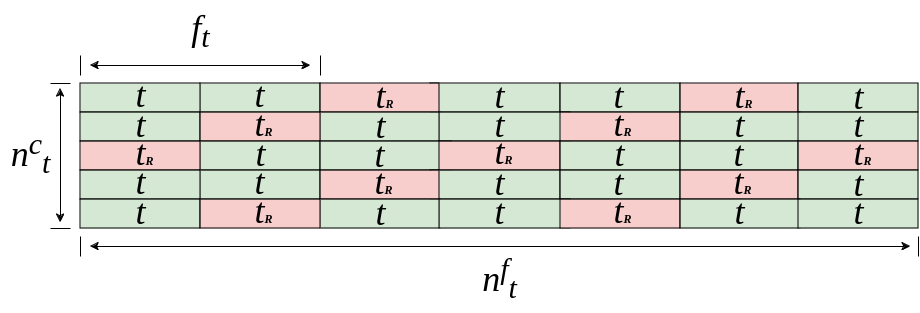}
    \subcaption{Full matrix pattern}
    \label{subfig:full_matrix_template}
  \end{minipage}
  
  
  \begin{minipage}[t]{\linewidth}
    \raggedleft
    \includegraphics[scale=0.27]{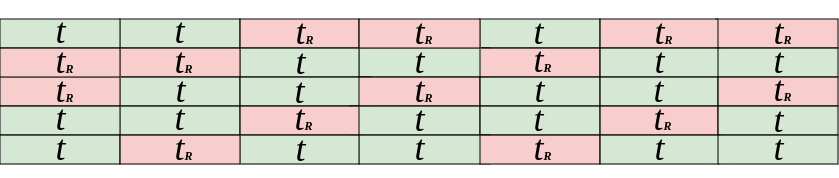}
    \subcaption{Corrected matrix pattern}
    \label{subfig:corrected_matrix_template}
  \end{minipage}
  
  
  \begin{minipage}[t]{\linewidth}
    \raggedleft
    \includegraphics[scale=0.27]{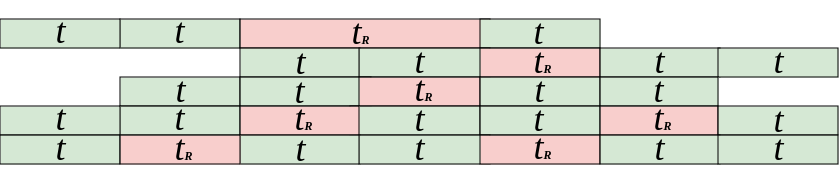}
    \subcaption{Reduced matrix pattern}
    \label{subfig:reduced_matrix_template}
  \end{minipage}
  
  
  \begin{minipage}[t]{\linewidth}
    \raggedleft
    \includegraphics[scale=0.27]{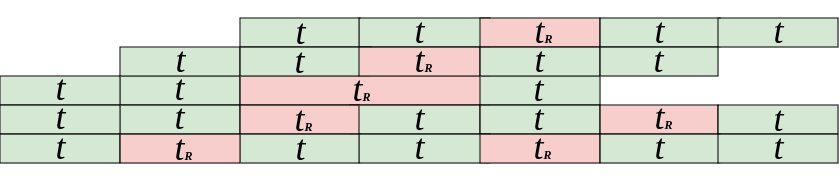}
    \subcaption{Sorted reduced matrix pattern}
    \label{subfig:sorted_reduced_matrix_template}
  \end{minipage}
  
  \caption{From top to bottom, the steps to build the robot allocation pattern for a relayable task. Example with $f_t=2$, $n^f_t=7$, $n^c_r=5$, and $n^r_t=3$.}
  \label{fig:template_matrix}
\end{figure}

In summary, the robot selection in Algorithm~\ref{alg:RobotSelection} differs for relayable tasks as follows. In line 5, instead of picking the first $n_t^r$ robots, the  matrix pattern described above is built and a compatible robot is selected for each row. This is done by sorting robots according to their finish time ($T^f_r$) and matrix rows according to their initial time, so that robots with later finish time match rows that start later. After this, within the loop, the waiting times for each selected robot and the pre-recharge flags (in case the robot needs a recharge before starting its row) are recomputed until there are no more changes in the final $T^f_t$, but without varying the set of selected robots. More specifically, line 14 becomes $\tilde{T}^f_t \gets \max_{r \in \mathcal{R}_t} T^f_r$, the condition of line 15 would be $\tilde{T}^f_t \neq T^f_t$, and line 16 would just be $T^f_t \gets \tilde{T}^f_t$. Finally, note that in line 15 of Algorithm~\ref{alg:HeuristicPlanner}, for the case of a relayable task, procedure \texttt{addTaskToQueue} will add to each of the selected robots' queues the corresponding row of the matrix pattern as a whole, including a pre-recharge operation at the beginning if indicated by $M_R$.

\section{Mission Replanning and Execution}
\label{sec:online_replaning}

\begin{figure}[tb]
  \centering
  \includegraphics[width=0.75\columnwidth]{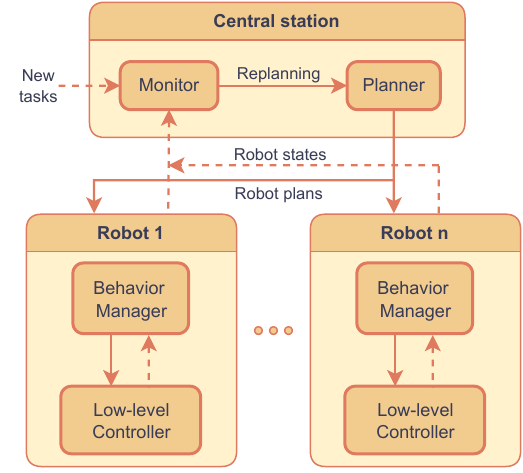}
  \caption{System architecture for mission (re-)planning and execution. The components related to task planning (top layer) run on a central station, whereas those in charge of task execution (bottom layer) are allocated on board each robot. Dashed lines depict feedback information and solid lines indicate action requests.}
  \label{fig:system_architecture}
\end{figure}

In this work, we deal with dynamic scenarios where during the execution of the mission, the planning conditions may deviate from their initial states. For instance, a robot may become delayed or simply fail while performing its tasks, or new task requests could arrive during the mission. In these cases, the running plan may not be valid anymore, making a replanning procedure necessary to adapt to the new circumstances. We integrate our \ac{MRTA} algorithms into a mission planning and execution framework that is robust to dynamic settings, allowing online replanning due to unexpected events. The system architecture, depicted in Figure~\ref{fig:system_architecture}, was presented in our previous work~\cite{calvo_icuas22}. The components are separated into two interleaved layers; one for mission (re-)planning and another for mission execution.  

Mission planning (top layer in Figure~\ref{fig:system_architecture}) is run in a centralized manner; given the current state of the scenario (i.e., the pending tasks and the location and battery status of the available robots), the \emph{Planner} computes an optimal plan for the team (Section~\ref{sec:planner}), deciding on the best task allocation. After this, each robot receives its plan, which consists of a schedule with its assigned tasks, and mission execution (bottom layer in Figure~\ref{fig:system_architecture}) is run in a distributed fashion. Each robot runs its own \emph{Behavior Manager} onboard, which is an executive component that, for each assigned task, extracts the task parameters and activates the \emph{Low-level Controller} to handle execution from a robot control point of view. More specifically, each Behavior Manager implements a state machine encoded as a~\ac{BT}~\cite{Colledanchise2017TRO}, which monitors task outcomes (whether execution has finished successfully or not) and the robot state (whether it has failed or become delayed). If a robot failure is detected, the~\ac{BT} activates contingency actions (e.g., an emergency landing in the case of a~\ac{UAV} running out of battery) and reports the robot's unavailability to the Monitor.\footnote{In our implementation, failures are a robot running out of battery or losing connectivity due to a communication dropout, but other hardware issues could easily be accommodated.} In the absence of failures, the Behavior Manager activates the Low-Level Controller, which takes care of navigational actions and collision avoidance so that the robot can execute its task. Note that in multi-robot tasks, the Low-Level Controllers involved may need to share additional information for robot coordination. Depending on the task being executed, different behaviors will be implemented in the Low-Level Controller. For instance, an inspection task may require the robot to activate a camera and navigate through a series or waypoints, while a delivery task may involve pick-up and place actions. More details about the implementation of Low-Level Controllers and~\acp{BT} for a multi-\ac{UAV} inspection application can be consulted in our previous work~\cite{calvo_icuas22}.

The multi-robot plan execution is centrally supervised by the \emph{Monitor} component at all times. This component receives feedback from the Behavior Managers indicating task/robot failures or delays. The Monitor also receives an input signal with new incoming tasks in the scenario.\footnote{Note that a task that is not successfully finished can be considered as a new task to be re-allocated again in future plans.} In case of a robot failure, if the running plan remains valid (i.e., satisfies all problem constraints in Section~\ref{sec:milp}), no action is taken. This can occur if the robot has no assigned tasks or belongs to a robot coalition without a \emph{fixed} size (soft constraint). In the case of a robot delay, the remaining plan may no longer comply with the problem constraints regarding battery life or time coordination, and a repair operation is attempted to adapt robot schedules and maintain all constraints. Thus, a new whole plan is only required in the following circumstances: 1) the arrival of new tasks; 2) a robot failure leading to an invalid plan; or 3) an unsuccessful plan repair operation. The Monitor carries out this replanning for all pending tasks and available robots using the Planner module (Section~\ref{sec:planner}). Tasks already in progress are not interrupted and reallocated unless their coalition size becomes insufficient for completion.

\subsection{Plan repair}

During plan execution, every time a robot finishes a task, we check whether it is delayed; that is, whether the robot will reach the starting position of its next allocated task at a different time instant than originally planed. This could happen because the robot took longer (or less time) than expected while performing the task, or because it ended up in a different position than expected, with the consequent variation in its arrival time to its next task. 
However, sometimes the whole plan could be repaired by \emph{delaying} the remaining tasks and still be a valid plan. The core idea is the following: the robot could accommodate its delay through tasks with associated waiting time by reducing their waiting times accordingly. Thus, for each future task, we accommodate part of the robot delay by updating its waiting time and delaying the corresponding start/finish task time instants accordingly. Then we propagate the remaining delay forward throughout the rest of the plan (see Algorithm~\ref{alg:updateTimeVars}). Any time this delay propagation reaches a time coordination point; i.e., a multi-robot task with several robots starting synchronously or a task where a relay is carried out, all the robots involved also need to update their schedules to comply with time coordination constraints (see Section~\ref{sec:timeCoordination}). Algorithms~\ref{alg:updateSynchTask} and~\ref{alg:updateRelayTask} are in charge of updating robots' plans to resolve the two previous time coordination cases. After a time coordination point is resolved, the delay by one of the robots involved may affect others' plans, and then those robots will also need to propagate forward their new delayed plans. The whole repair procedure is carried out by Algorithm~\ref{alg:repairPlans}, which sweeps the timeline of the multi-robot plan propagating delayed robot schedules, resolving coordination points as they appear. The primary objective is to restore plan validity without task reassignment, and this is done by minimizing schedule extensions (i.e., optimizing makespan).

\begin{algorithm}
  \footnotesize
  \caption{{\sc {UpdateTimeVars\,($r,s_f,\bar{s},\Delta t,\mathcal{R}_d$)}}}\label{alg:updateTimeVars}
  \begin{algorithmic}[1]
  
  \If{$\bar{s}[r] \ge s_f$} \Return
  \EndIf  
  
  \State $s \gets \bar{s}[r]$ , $\delta \gets \Delta t[r]$ 
  
  \While{$s < s_f$ \textbf{and} $\delta > 0$}
    \State $s \gets s+1$
  
    \If{$T^w_{r,s} > 0$}
      \If{$T^w_{r,s} > \delta$}
        \State $T^w_{r,s} \gets T^w_{r,s} - \delta$
        \State $\delta \gets 0$
      \Else
        \State $\delta \gets \delta - T^w_{r,s}$
        \State $T^w_{r,s} \gets 0$
      \EndIf
    \EndIf
  
    \State $T^f_{r,s} \gets T^f_{r,s} + \delta$ 

  \EndWhile

  \State $\bar{s}[r] \gets s_f$ , $\Delta t[r] \gets \delta$

  \If{$\Delta t[r] == 0$} $\mathcal{R}_d$.\texttt{remove}($r$)
  \EndIf

  \end{algorithmic}
  \end{algorithm}

Algorithm~\ref{alg:updateTimeVars} receives as input a robot $r$ and updates its schedule, propagating its delay up to a given slot $s_f$. The algorithm also receives three global variables that may have to be updated at any time: $\bar{s}$ is a vector that indicates the latest slot updated so far for each robot's schedule, $\Delta t$ is a vector with the current delay for each robot's schedule, and $\mathcal{R}_d$ is a list that contains all the robots with delayed plans at each moment. If the robot plan still has slots to be updated (line 1), the algorithm iterates over those slots until the target slot $s_f$ is reached or there is no remaining delay to consume (lines 3--12). If a slot has waiting time assigned, this is reduced to accommodate the remaining delay either completely (lines 6--8) or partially (lines 9--11). The final time for the task allocated to the slot is correspondingly delayed (line 12). Before finishing, the latest slot processed for the robot schedule and its current delay are updated (line 13). If the robot has managed to accommodate all its delay within its waiting periods, its plan is no longer delayed and it is removed from the list $\mathcal{R}_d$ (line 14).

\begin{algorithm}
  \footnotesize
  \caption{{\sc {UpdateSynchTask\,($\mathcal{A}^+,\bar{s},\Delta t,\mathcal{R}_d$)}}}\label{alg:updateSynchTask}
  \begin{algorithmic}[1]
  
  \For{$(r,s)$ in $\mathcal{A}^+$}

    \State $\delta \gets \Delta t[r]$ 

    \If{$r \in \mathcal{R}_d$}

      \If{$T^w_{r,s} > \delta$} \Comment{Accommodate delay fully}
        \State $T^w_{r,s} \gets T^w_{r,s} - \delta$
        \State $\delta \gets 0$
        
      \Else \Comment{Push forward task end}
        \State $\delta \gets \delta - T^w_{r,s}$
        \State $T^w_{r,s} \gets 0$
        \State $T^f_{r,s} \gets T^f_{r,s} + \delta$ 
  
      \EndIf
    \EndIf

    \State $\bar{s}[r] \gets s$ , $\Delta t[r] \gets \delta$
    
  \EndFor

  \State $(r_{max},s_{max}) \gets \argmax_{(r,s) \in \mathcal{A}^+} \; \Delta t[r]$
  \State $\delta_{max} \gets \Delta t[r_{max}]$

  \If{$\delta_{max} == 0$} \Comment{All robots synchronized}

    \For{$(r,s)$ in $\mathcal{A}^+$} $\mathcal{R}_d$.\texttt{remove}($r$) 
    \EndFor
  \Else

    \For{$(r,s)$ in $\mathcal{A}^+$}  \Comment{Synch all robots with latest}

      \State $\delta \gets \Delta t[r]$

      \State $T^w_{r,s} \gets T^w_{r,s} + \delta_{max} - \delta$
      \State $T^f_{r,s} \gets T^f_{r,s} + \delta_{max} - \delta$
      \State $\Delta t[r] \gets \delta_{max}$
      \State $\mathcal{R}_d$.\texttt{add}($r$)
    \EndFor
  \EndIf

  \end{algorithmic}
\end{algorithm}

Algorithm~\ref{alg:updateSynchTask} updates the schedules of a set of robots involved in the first type of coordination point; i.e., a coalition that has to start a multi-robot task synchronously. The algorithm receives $\mathcal{A}^+$ as input, which is a list of pairs $(r,s)$ indicating 1) the identifiers of the robots involved in the synchronization, and 2) the corresponding slots within their schedules in which the coordinated task occurs. All schedules are assumed to be updated (delay propagation) up to their previous slot $s-1$, so only the time variables corresponding to the slots where the coordination takes place need to be updated. First, all delayed robots in the coalition are checked (lines 1--11) to accommodate their delays within their waiting period either completely (lines 4--6) or partially (lines 7--10). Then the robot with maximum remaining delay in the coalition is computed (lines 12--13). If all delays have been fully accommodated, the whole coalition is already synchronized (lines 14--15). Otherwise, there are still robots with remaining delay that will push the end of the multi-robot task forward, and since all must synchronize their finish times, the robot with maximum delay is taken as reference. All the other robots increase their waiting time (and their task finish time) to match the \emph{latest} robot (lines 16--22).  

\begin{algorithm}
  \footnotesize
  \caption{{\sc {UpdateRelayTask\,($^-\!\!\mathcal{A},\mathcal{A}^+,\bar{s},\Delta t,\mathcal{R}_d$)}}}\label{alg:updateRelayTask}
  \begin{algorithmic}[1]

  \For{$(r,s)$ in $\mathcal{A}^+$}

    \State $\delta \gets \Delta t[r]$ 

    \If{$r \in \mathcal{R}_d$}

      \If{$T^w_{r,s} > \delta$} \Comment{Accommodate delay fully}
        \State $T^w_{r,s} \gets T^w_{r,s} - \delta$
        \State $\delta \gets 0$
        
      \Else \Comment{Push forward task end}
        \State $\delta \gets \delta - T^w_{r,s}$
        \State $T^w_{r,s} \gets 0$
        \State $T^f_{r,s} \gets T^f_{r,s} + \delta$ 
        
      \EndIf
    \EndIf

    \State $\bar{s}[r] \gets s$ , $\Delta t[r] \gets \delta$
    
  \EndFor

  \State $(r_{max},s_{max}) \gets \argmax_{(r,s) \in ^-\!\!\mathcal{A},\mathcal{A}^+} \; \Delta t[r]$
  \State $\delta_{max} \gets \Delta t[r_{max}]$

  \If{$\delta_{max} == 0$}

    \For{$(r,s)$ in $\mathcal{A}^+$}  \Comment{All robots synchronized}
      \State $\mathcal{R}_d$.\texttt{remove}($r$)
    \EndFor

  \Else

    \For{$(r,s)$ in $^-\!\!\mathcal{A}, \mathcal{A}^+$} \Comment{Synch all robots with latest}
      \State $\delta \gets \Delta t[r]$
      \State $T^w_{r,s} \gets T^w_{r,s} + \delta_{max} - \delta$
      \State $T^f_{r,s} \gets T^f_{r,s} + \delta_{max} - \delta$
      \State $\Delta t[r] \gets \delta_{max}$
      \State $\mathcal{R}_d$.\texttt{add}($r$)
    \EndFor

  \EndIf

  \end{algorithmic}
\end{algorithm}

Algorithm~\ref{alg:updateRelayTask} updates the schedules of a set of robots involved in the second type of coordination point, a coalition (or single robot) that is relayed by another coalition (or single robot). This time the algorithm receives as input two lists, $^-\!\!\mathcal{A}$ and $\mathcal{A}^+$: the former contains pairs $(r,s)$ with the robots being relayed and the slots that are relayed in their corresponding schedules; the latter contains pairs $(r,s)$ with the relaying robots and their relaying slots. 
All schedules are assumed to be updated (delay propagation) up to the slot previous to the relay; note that this is $s$ for robots in $^-\!\!\mathcal{A}$ and $s-1$ for robots in $\mathcal{A}^+$. 
The algorithm starts accommodating all possible delay for the robots in $\mathcal{A}^+$ within their waiting period (lines 1--11), as was done in Algorithm~\ref{alg:updateSynchTask}. Then the robot with maximum remaining delay from the relayed and relaying sets is computed (lines 12--13). In the case of a relay, both the relayed and relaying robots must be time coordinated, and the \emph{latest} robot will determine the new relay time instant. If all delays become fully accommodated, the coordination point is already solved (lines 14--16). Otherwise, the relay instant is put off according to the maximum delay, and all robots involved increase their waiting time (and their task finish time) to match the \emph{latest} robot (lines 18--23). An example of this procedure is depicted in Figure~\ref{fig:updateRelayTask_example}. 

\begin{figure}[tb!]
  \centering
  \includegraphics[width=.7\columnwidth]{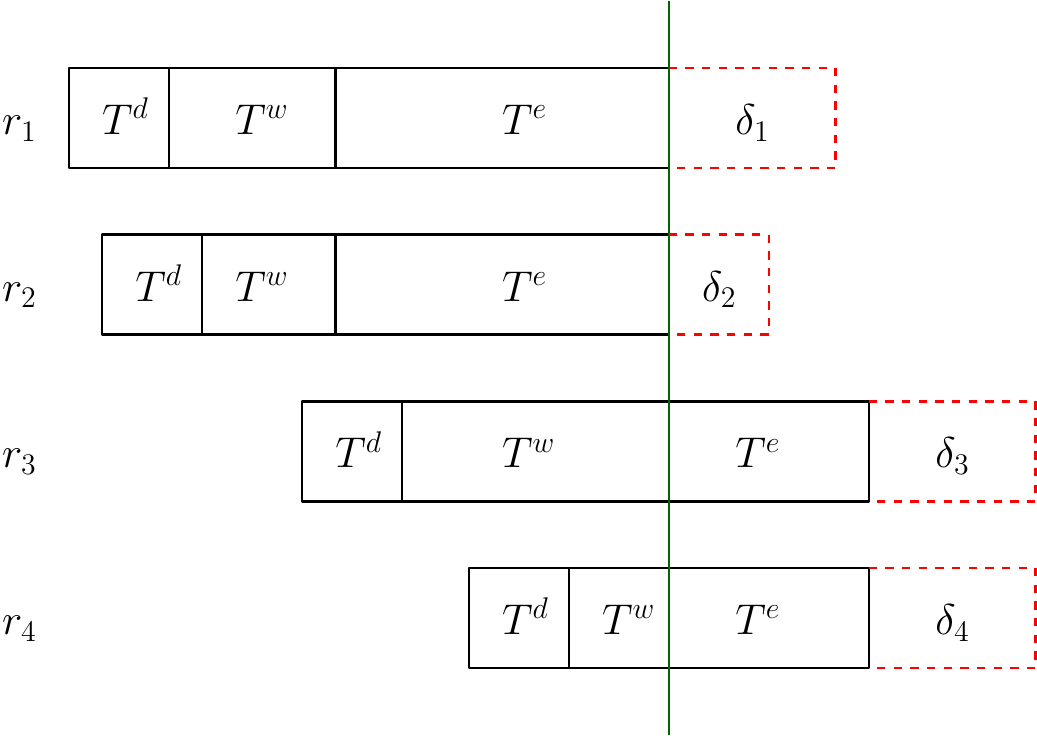}\\
  \includegraphics[width=.7\columnwidth]{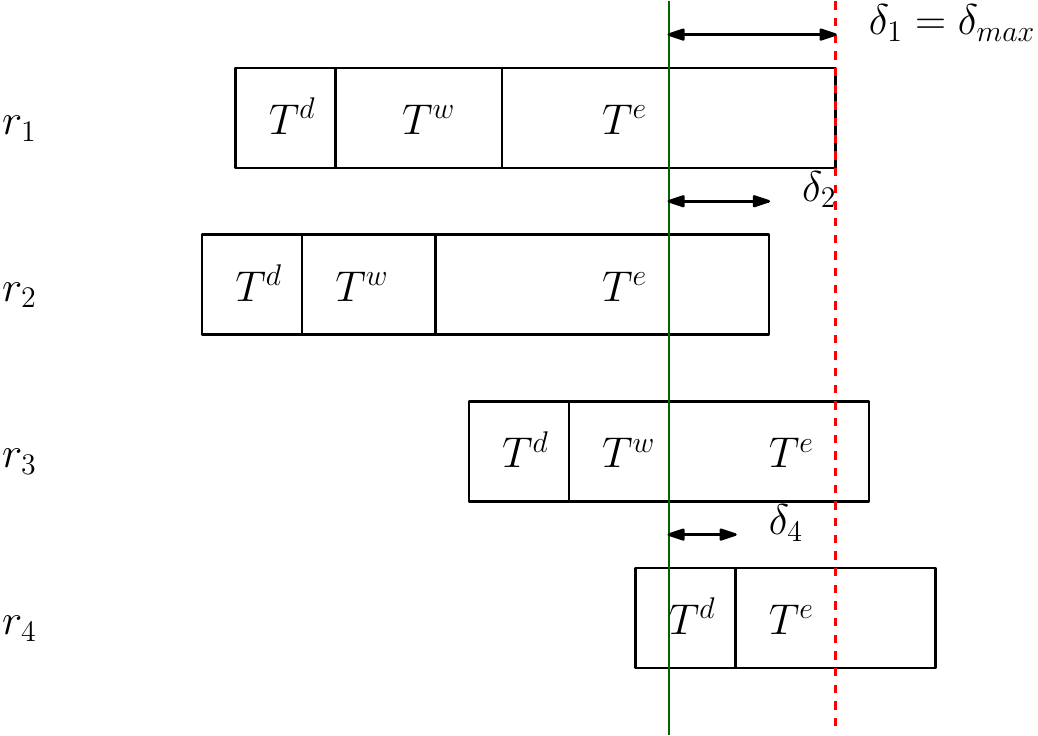}\\
  \includegraphics[width=.7\columnwidth]{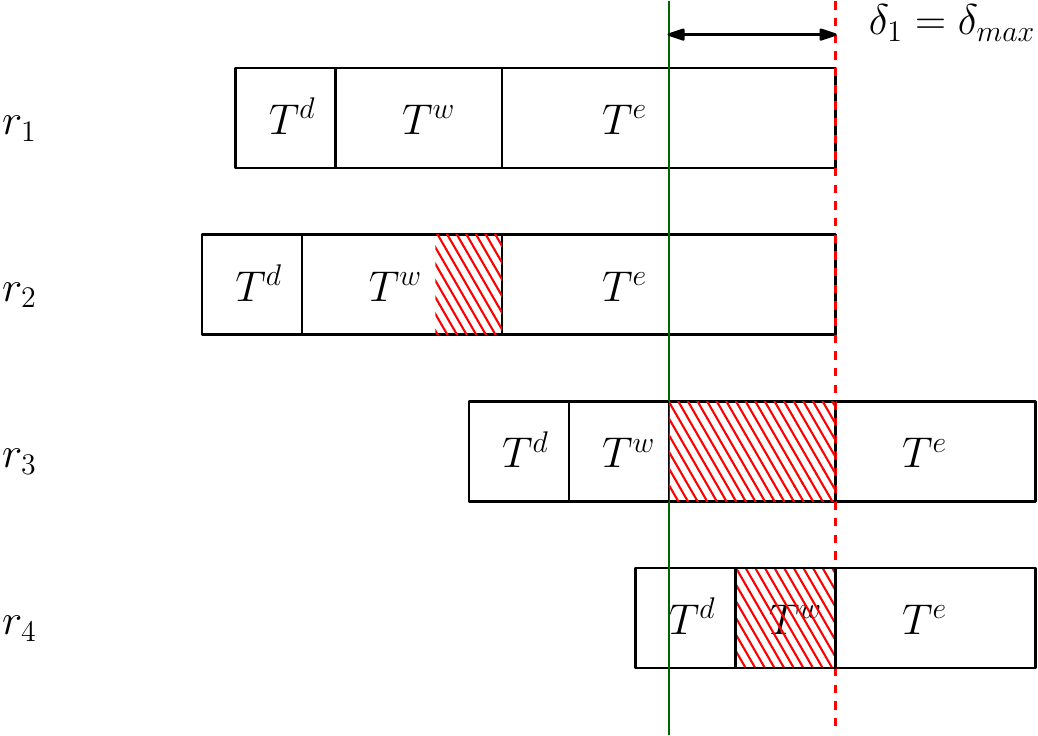}
  
  \caption{Example of Algorithm~\ref{alg:updateRelayTask} repairing the schedules of a relay point. Top view: originally, robots $r_1$ and $r_2$ ($^-\!\!\mathcal{A}$) execute a multi-robot task which is relayed by robots $r_3$ and $r_4$ ($\mathcal{A}^+$) at the time line in green. Dashed red squares represent the delay for each robot schedule. Middle view: first, the relays of robots $r_3$ and $r_4$ are partially accommodated by reducing waiting times; relays of robots in $^-\!\!\mathcal{A}$ ($r_1$ and $r_2$) were already propagated before Algorithm~\ref{alg:updateRelayTask}. The maximum delay ($r_1$) determines the new relay time instant (dashed red vertical line). Bottom view: robots synchronize with the new relay instant by extending their waiting periods (red patterns).}
  \label{fig:updateRelayTask_example}
\end{figure}

\begin{algorithm}
  \footnotesize
  \caption{{\sc {RepairPlans\,($t_0,r',\delta',\bar{s}$)}}}\label{alg:repairPlans}
  \begin{algorithmic}[1]
  
  \If{$\delta' <= 0$}

    \State $s' \gets \bar{s}[r']$ , $T^w_{r',s'+1} \gets T^w_{r',s'+1} + |\delta'|$
    \State \Return \texttt{True}

  \Else

    \State $\mathcal{R}_d \gets \emptyset$ , $\mathcal{R}_d$.\texttt{add}($r'$)
    \For{$r$ in $\mathcal{R}$}
      \State $\Delta t[r] \gets 0$ 
    \EndFor
    \State $\Delta t[r'] \gets \delta'$ 
    
    \State $ \mathcal{C} \gets$ \texttt{GetOrderedCoordinationTasks}($t_0$)

    \For{$(^-\!\!\mathcal{A},\mathcal{A}^+)$ in $\mathcal{C}$}

      \For{$(r,s)$ in $\mathcal{A}^+$}
        \State \texttt{UpdateTimeVars}($r,s-1,\bar{s},\Delta t,\mathcal{R}_d$)
      \EndFor

      \If{$^-\!\!\mathcal{A} \neq \emptyset$}
      
        \For{$(r,s)$ in $^-\!\!\mathcal{A}$}
          \State \texttt{UpdateTimeVars}($r,s,\bar{s},\Delta t,\mathcal{R}_d$)
        \EndFor
        
        \State \texttt{UpdateRelayTask}($^-\!\!\mathcal{A},\mathcal{A}^+,\bar{s},\Delta t,\mathcal{R}_d$)
      \Else

        \State \texttt{UpdateSynchTask}($\mathcal{A}^+,\bar{s},\Delta t,\mathcal{R}_d$)

      \EndIf
      \If{$\mathcal{R}_d == \emptyset$} \textbf{break}
      \EndIf
    \EndFor

    \For{$r$ in $\mathcal{R}_d$}
      \State \texttt{UpdateTimeVars}($r,|\mathcal{S}|,\bar{s},\Delta t,\mathcal{R}_d$)
    \EndFor

    \State $valid \gets$ \texttt{True}

    \For{$r$ in $\mathcal{R}$}
      \State $valid \gets valid \; \cap$ \texttt{checkPlanBattery}($r$)
    \EndFor

    \State \Return $valid$
  \EndIf

  \end{algorithmic}
\end{algorithm}

Algorithm~\ref{alg:repairPlans} attempts to repair a multi-robot plan given that one of the robots $r'$ has finished its assigned task at $t_0$ and has a delay $\delta'$ before its next scheduled task. The algorithm receives as input the vector $\bar{s}$, which indicates the current running slot for each robot's schedule.
If the delay is negative, the robot is ahead of its schedule and a waiting period for its next slot is added for synchronization (lines 1--4). Otherwise, after initializing variables (lines 5--8), a time ordered list $\mathcal{C}$ with all the coordination points in the multi-robot schedule starting at $t_0$ is computed (line 9). Each item in $\mathcal{C}$ is a pair $(^-\!\!\mathcal{A},\mathcal{A}^+)$ with all the information on the robots involved in the corresponding coordination point. The algorithm synchronizes robots' schedules for all coordination points until there are no more delayed robots (lines 10--19). Depending on the type of coordination point, Algorithm~\ref{alg:updateSynchTask} or Algorithm~\ref{alg:updateRelayTask} is used. Before each coordination point is synchronized, Algorithm~\ref{alg:updateTimeVars} propagates delays in the schedules of the involved robots up to the slot preceding the coordination point. Given the definitions of $\mathcal{A}^+$ and $^-\!\!\mathcal{A}$, this will be up to $s_f=s-1$ or $s_f=s$, respectively. Finally, after all coordination points are swept, all robots' schedules are delay propagated up to their last slot ($s_f=|\mathcal{S}|$) and the resulting plans are checked for battery compliance (lines 20--24). If the repaired plan does not comply with battery constraints, the algorithms reports a failure.



\section{Experimental Results}
\label{sec:experimental_results}

This section presents experimental results that evaluate our methods. 
Given the use case in Section~\ref{sec:use_case} and the experimental setup in Section~\ref{sec:experimental_setup}, we present numerical experiments (Sections~\ref{sec:small_scenarios_exp} and~\ref{sec:scalability_exp}) with a twofold objective: (i) evaluate optimal solutions and demonstrate the potential of our \ac{MILP} formulation to solve more complex scenarios, due to our higher degree of freedom in terms of task decomposability and coalition size flexibility; and (ii) assess our heuristic planner performance and test its scalability for larger scenarios. 
Furthermore, we present a realistic simulation to demonstrate the feasibility of our approach in real applications and show the advantages of our replanning framework in dynamic scenarios (Section~\ref{sec:replanning_exp}).

\subsection{Use case description}
\label{sec:use_case}

We define a use case taking as inspiration our running example in Section~\ref{sec:problem_description}: a team of \acp{UAV} that provides support to human workers during inspection and maintenance operations in a solar energy plant. We used an actual photovoltaic facility (see Figure~\ref{fig:evora}) of size approximately $200\times300$~m located in Evora, Portugal. 
The~\acp{UAV} can help with three types of tasks: 1) \emph{inspection}, in which they capture visual or thermal images of a given remote area with solar panels; 2) \emph{monitoring}, in which they provide the supervising team with a view of a human worker operating on the plant, in order to assess his/her safety; and 3) \emph{delivery}, in which they transport and hand out items such as tools or parts to a human worker. According to our categorization in Section~\ref{sec:problem_description}, inspection tasks are fragmentable (the target area could be divided), monitoring tasks are relayable (a continuous videostream of the worker is required in risky operations), and delivery tasks are non-decomposable. 
The team is heterogeneous: depending on its payload, each \ac{UAV} can conduct certain tasks. Specific cameras enable inspection and monitoring tasks, while a load transportation system enables delivery tasks.
The maximum flight time for each vehicle is $B_r^{max}=20$ minutes and its traveling speed $v_r=5$~m/s. The recharging station can be used by multiple vehicles simultaneously and is located at the central point of the plant.
Each recharge operation takes $T^e_{t_R} = 5$ minutes.

\begin{figure}[tb]
  \centering
  \includegraphics[width=0.65\columnwidth]{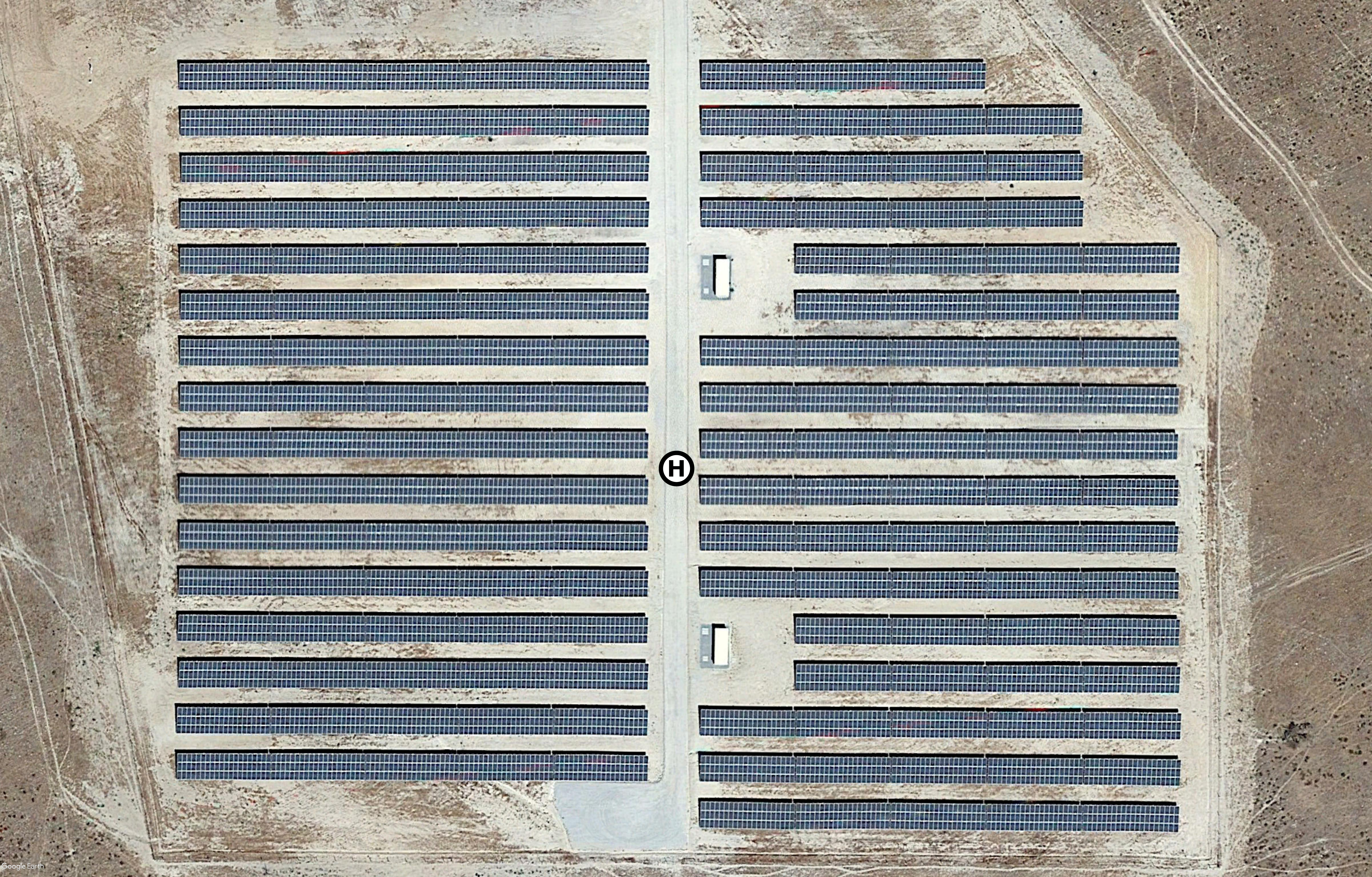}
  \caption{Photovoltaic solar plant in Evora (Portugal). A single recharging station is located in the middle (circle mark).}
  \label{fig:evora}
\end{figure}

\subsection{Experimental setup}
\label{sec:experimental_setup}

We implemented our code in MATLAB R2022b\footnote{Code at~\url{https://github.com/multirobot-use/mrta_heuristic_planner}.} using Gurobi to compute exact optimal solutions.
All experiments were run with an Intel 8-core i7-7700 CPU@3.60GHz with 15.6GB RAM.
The methods were evaluated over a set of random scenarios based on the use case in Section~\ref{sec:use_case}. \acp{UAV} start at random positions within the solar plant, with an initial consumed battery time $B_{r,0}$ uniformly sampled from $0$, $0.25$, or $0.5$ of the total flight time $B_r^{\maxs}$.
Tasks are also placed at random positions with a deadline $T_t^{\maxs}=100$~minutes and an execution time $T_t^e$ uniformly sampled from $0.35$, $1.25$, or $2.5$ of the total flight time (this duration is set to $0.35 \cdot B^{\maxs}_r$ for tasks that cannot be fragmented, so they do not last longer than the total flight time).
Each task has a decomposability uniformly sampled from \{\emph{non-decomposable}, \emph{fragmentable}, \emph{relayable}\} and a coalition size flexibility sampled from \{\emph{fixed}, \emph{variable}, \emph{unspecified}\}. The specified coalition size $N_t$ is sampled from 1 to the maximum number of compatible~\acp{UAV}. 
The hardware capabilities for each \ac{UAV} are uniformly sampled from three different types. Given the set of sampled \ac{UAV} types in the team, each task is then randomly set as compatible or incompatible (with a $0.5$ probability) for each of these hardware types, ensuring compatibility with at least one type.

In order to evaluate the multi-robot plans for each method in a given set of scenarios, we defined the following metrics:

\begin{enumerate}
    \item \emph{Success Rate (SR)}: Percentage of scenarios for which a solution is found.
  	\item \emph{Recharge Rate (RR)}: Percentage of the solved scenarios that have at least one recharge.
 	  \item \emph{Number of Recharges (NR)}: Total number of recharges in a solution.
 	  \item \emph{Objective function value ($f$)}: Value of the function~\eqref{eq:costFunction} for a solution. This is a cost to be minimized.
 	  \item \emph{Makespan ($Z$)}: Makespan, time by which the latest robot finishes its plan.
 	  \item \emph{Waiting Time Rate (WTR)}: Percentage of time each robot is waiting out of its whole plan duration, averaged over all robots in the plan. 
    \item \emph{Coalition Size Deviation (CSD)}: Relative coalition size deviation for multi-robot tasks $V_t/N_t$, averaged for all tasks in the plan with a \emph{variable} coalition size. 
 	  \item \emph{Consumed Battery Time (CBT)}: Average battery time consumed per robot in a plan, taking into account that the execution and waiting time during recharge tasks do not consume battery time. 
 	  \item \emph{Workload Distribution (WD)}: Percentage of time each robot's plan lasts relative to the whole mission duration (makespan), averaged over all robots in the plan. If the work were equally distributed, all robots would finish simultaneously at the makespan, and this metric would be 100\%. Therefore, the higher its value the better. 
 	  \item \emph{Computation Time (CT)}: Time needed to compute a solution. 
\end{enumerate}

\subsection{Small-scale scenarios}
\label{sec:small_scenarios_exp}

\begin{figure*}[tbh!]
  \centering
  
    \includegraphics[width=.25\linewidth]{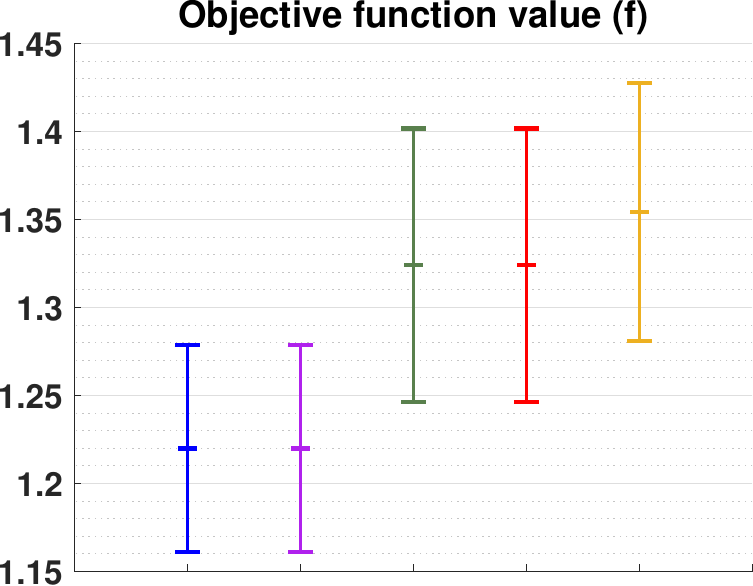}
    \includegraphics[width=.25\linewidth]{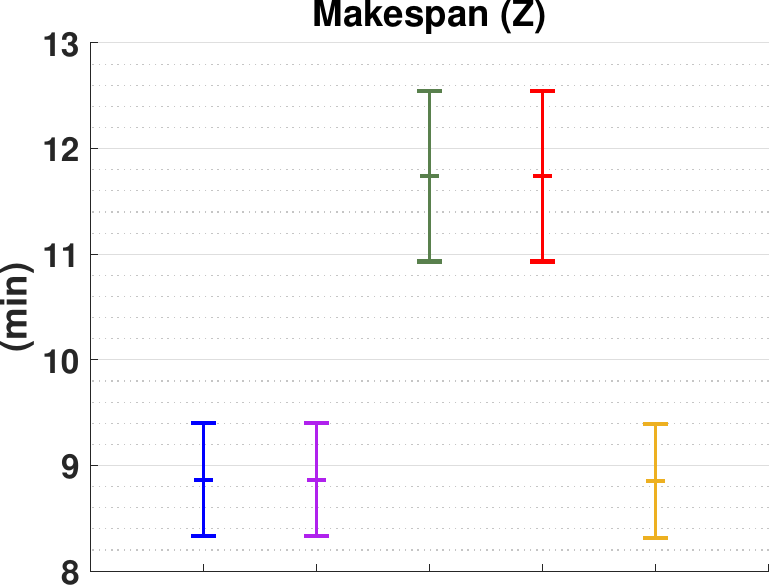}
    \includegraphics[width=.25\linewidth]{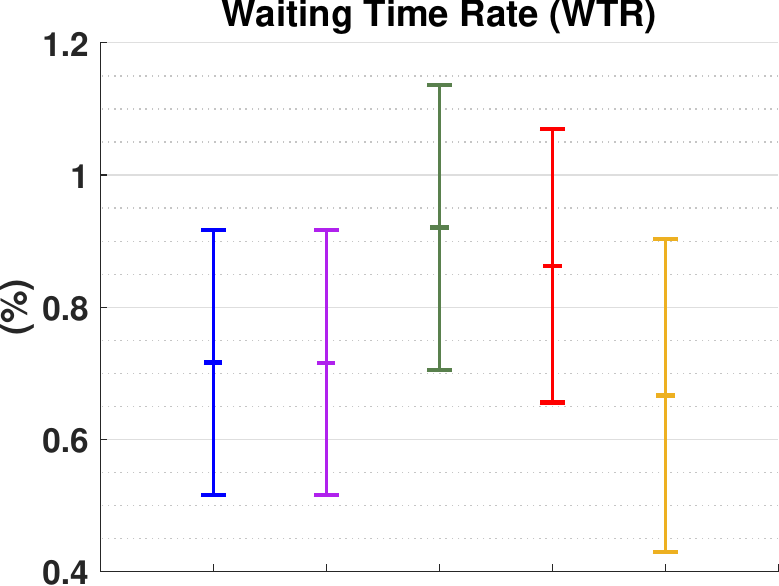}\\
    \includegraphics[width=.25\linewidth]{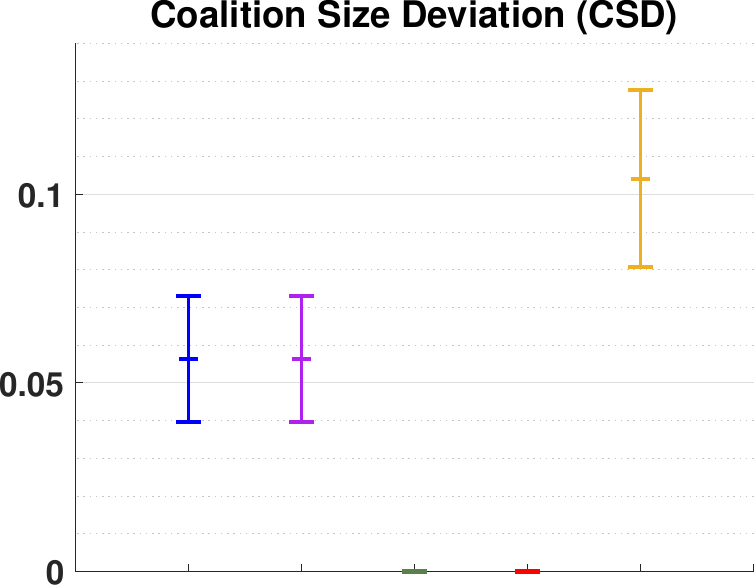}
    \includegraphics[width=.25\linewidth]{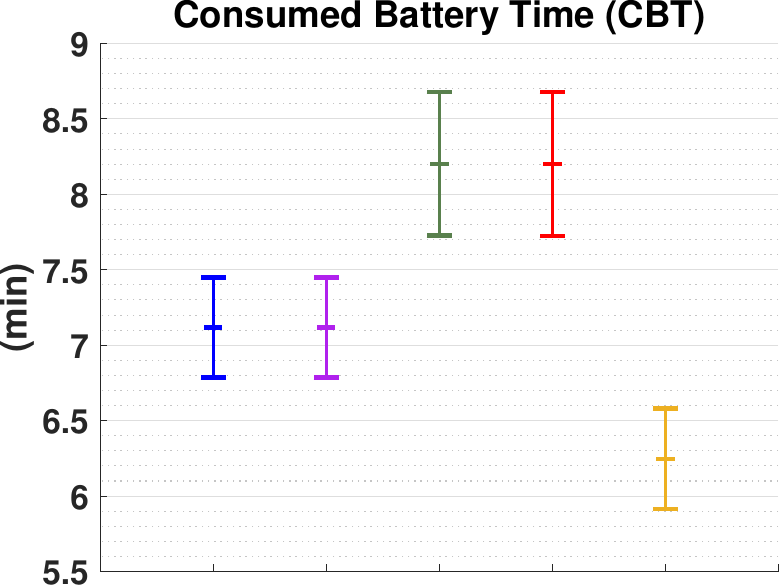}
    \includegraphics[width=.25\linewidth]{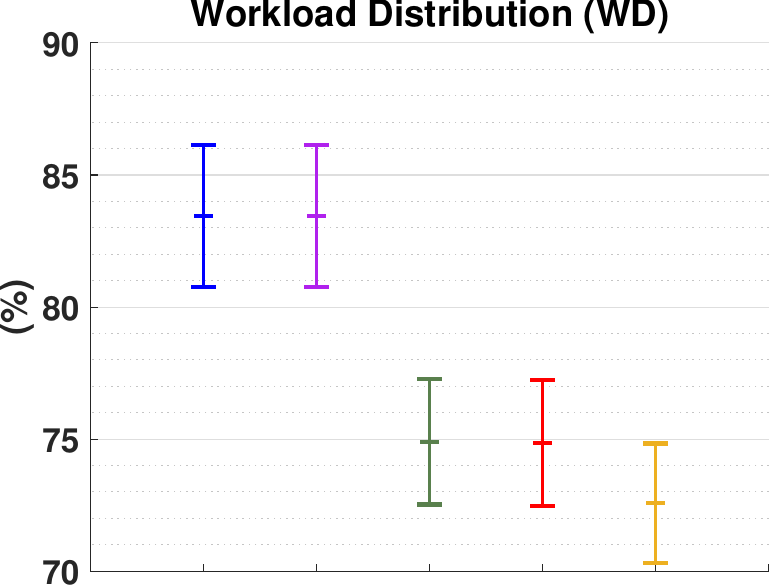}\\
    \includegraphics[width=.6\linewidth]{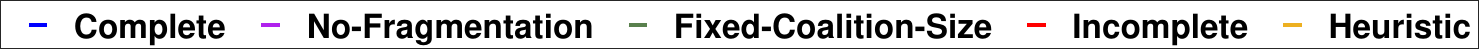}
    
  \caption{Resulting metrics for the small-scale scenarios. Mean and standard deviation are shown for the scenarios that are solvable with the five approaches being compared.}
  \label{fig:small_scale_metrics}
\end{figure*}

In this section, we show numerical experiments to demonstrate the potential of our \ac{MILP} formulation. We prove that by including features such as task decomposition and coalition size flexibility, our approach finds feasible solutions for scenarios that would be unsolvable otherwise. We evaluate the quality of optimal solutions when certain features are removed from the formulation, in order to analyze their influence on the results. Lastly, we compare the performance of our heuristic solver (Section~\ref{sec:planner}) against exact optimal solutions.  

We generated 100 random scenarios following the methodology explained in Section~\ref{sec:experimental_setup}. Since the number of variables in our~\ac{MILP} formulation increases significantly with the size of the multi-robot team and the number of tasks in the mission, numerical solvers such as Gurobi do not scale with large scenarios. Therefore, we limited this first experiment to small-scale scenarios where Gurobi was able to find solutions without running out of memory. In particular, the scenarios created had missions with 2 tasks (plus the required recharge tasks) and 3 available~\acp{UAV}. Due to the size of the scenarios, we did not limit the hardware capabilities of the robots, so all~\acp{UAV} were able to execute any of the tasks. Although small, these scenarios exhibit enough complexity to yield plans with a rich variety of features such as recharges, task fragmentation/relays, and multi-robot synchronization. Therefore we believe they can be useful for drawing some conclusions about the characteristics of optimal solutions.\footnote{Video with illustrative examples at: \url{https://youtu.be/a6zXwTZKFR4}} 

We compared different approaches to solve the scenarios: 1) \emph{Complete}: optimal plans for our complete \ac{MILP} formulation~\eqref{eq:MILP}; 2) \emph{No-Fragmentation}: optimal plans for our~\ac{MILP} formulation considering all tasks non-decomposable; 3) \emph{Fixed-Coalition-Size}: optimal plans for our~\ac{MILP} formulation considering all tasks with fixed coalition size (undefined tasks are always assigned to a single robot in this variant); 4) \emph{Incomplete}: optimal plans for our~\ac{MILP} formulation combining Fixed-Coalition-Size and No-Fragmentation variants; 5) \emph{Heuristic}: plans obtained with our heuristic solver.
Figure~\ref{fig:small_scale_metrics} shows the results. Makespan, WTR, and CSD are explicitly minimized in our \ac{MILP} formulation through the objective function $f$. The Complete variant outperforms the Incomplete one across all these metrics except for CSD, which is inherently zero in variants with fixed coalition sizes. The Complete variant also performs better than the Incomplete variant for metrics not explicitly optimized, namely CBT and WD. This demonstrates that leveraging task fragmentation and coalition size flexibility improves plan quality for solvable scenarios.
Similar results are observed for the Complete and No-Fragmentation variants, and for the Fixed-Coalition-Size and Incomplete variants, respectively. This might suggest that fragmentation offers no advantage. However, the SR was $100\%$ for the Complete variant, compared to $37\%$ for No-Fragmentation, $92\%$ for Fixed-Coalition-Size, $36\%$ for Incomplete, and $100\%$ for Heuristic. This indicates that fragmentation significantly increases the number of solvable scenarios. The high SR for Fixed-Coalition-Size stems from our generation of scenarios predominantly solvable with fixed coalition sizes. Note that the metrics in Figure~\ref{fig:small_scale_metrics} are calculated only for scenarios solvable by all approaches, meaning fragmentation was not strictly required. Nonetheless, fragmentation may still be beneficial even when not mandatory, as it allows task execution at different times, potentially improving battery utilization and reducing recharges. Such situations did not arise in these small-scale scenarios. Regarding recharges, the Complete and No-Fragmentation variants had an RR of $5.56\%$ with an average NR of $0.08$ per scenario; Fixed-Coalition-Size and Incomplete had RR of $16.67\%$ and average NR of $0.28$; and Heuristic had RR of $5.56\%$ and average NR of $0.06$. This suggests that the Complete version also achieves more efficient solutions than the Incomplete version by requiring fewer recharges, and the Heuristic solver performs similarly to the Complete version in terms of recharges.

Overall, heuristic solutions show worse performance in terms of objective function value $f$, as expected. Nonetheless, the heuristic solver found the optimal solution (i.e., a value for $f$ equal to the one in the Complete solution) in a $25\%$ of the scenarios. In addition, it is important to note that the Heuristic approach performed as the Complete in makespan, which is the most critical optimization criterion, and even outperformed the others in CBT and WTR.

\subsection{Scalability test}
\label{sec:scalability_exp}

\begin{table}[tbh!]
	\centering
  \caption{\ac{MILP} formulation size as the number of robots/tasks ($n/m$) increases. Values shown represent the worst case from 100 randomly generated scenarios. Scenarios below the horizontal line could not be solved with Gurobi.}
	\begin{adjustbox}{width=\columnwidth}
    \begin{tabular}{c|c c c c c c}
		  \hline
          \makecell{Scenario\\size} & 
          \makecell{Basic\\variables } & 
          \makecell{Time\\related\\variables } & 
          \makecell{Time\\coordination\\variables } & 
          \makecell{Linearization\\variables (\%) } & 
          \makecell{Integer / Real\\variables (\%) } & 
          \makecell{Total\\variables } \\
          \hline
          n=1 / m=1   & 34    & 74     & 77        & 61.08 & 56.76 / 43.24 & 185       \\
          n=1 / m=2   & 60    & 212    & 534       & 69.98 & 46.4 / 53.6   & 806       \\
          n=2 / m=1   & 66    & 292    & 1,040      & 69.24 & 39.41 / 60.59 & 1,398      \\
          n=2 / m=2   & 107   & 564    & 3,644      & 71.47 & 36.41 / 63.59 & 4,315      \\
          n=2 / m=3   & 176   & 1,156   & 12,208     & 72.08 & 34 / 66       & 13,540     \\
          n=2 / m=5   & 292   & 2,404   & 31704     & 72.45 & 33.34 / 66.66 & 34,400     \\
          n=3 / m=2   & 254   & 1,791   & 36,579     & 71.6  & 31.18 / 68.82 & 38,624     \\
          n=3 / m=3   & 335   & 2,598   & 61,468     & 71.8  & 31.06 / 68.94 & 64,401     \\
          \hline
          n=2 / m=10  & 771   & 10,144  & 189,896    & 72.53 & 32.17 / 67.83 & 200,811    \\
          n=5 / m=2   & 538   & 3,490   & 201,966    & 71.42 & 29.45 / 70.55 & 205,994    \\
          n=5 / m=5   & 1,824  & 16,670  & 2,102,357   & 71.54 & 29.08 / 70.92 & 2,120,851   \\
          n=10 / m=2  & 2,093  & 14,520  & 3,501,506   & 71.43 & 28.78 / 71.22 & 3,518,119   \\
          n=10 / m=10 & 29,851 & 479,720 & 423,639,104 & 71.45 & 28.65 / 71.35 & 424,148,675 \\
          \hline
          \hline
          \makecell{Scenario\\size} & 
          \makecell{Basic\\constraints } & 
          \makecell{Time\\related\\constraints } & 
          \makecell{Time\\coordination\\constraints } & 
          \makecell{Linearization\\constraints (\%) } & 
          \makecell{ - } & 
          \makecell{Total\\constraints} \\
          \hline
          n=1 / m=1   & 38    & 218     & 285        & 83.55 & - & 541 \\
          n=1 / m=2   & 81    & 662     & 2,030       & 87.41 & - & 2,773 \\
          n=2 / m=1   & 102   & 868     & 3,957       & 87.36 & - & 4,927 \\
          n=2 / m=2   & 173   & 1,764    & 13,970      & 88.41 & - & 15,907 \\
          n=2 / m=3   & 304   & 3,700    & 46,939      & 88.79 & - & 50,943 \\
          n=2 / m=5   & 512   & 7,804    & 122,051     & 88.99 & - & 130,367 \\
          n=3 / m=2   & 490   & 5,616    & 140,810     & 88.76 & - & 146,916 \\
          n=3 / m=3   & 637   & 8,322    & 236,749     & 88.86 & - & 245,708 \\
          \hline
          n=2 / m=10  & 1,417  & 32,764   & 731,826     & 89.11 & - & 766,007 \\
          n=5 / m=2   & 1,115  & 11,050   & 778,346     & 88.81 & - & 790,511 \\
          n=5 / m=5   & 3,752  & 54,890   & 8,107,103    & 88.89 & - & 8,165,745 \\
          n=10 / m=2  & 4,545  & 46,020   & 13,503,026   & 88.87 & - & 13,553,591 \\
          n=10 / m=10 & 61,697 & 1,549,820 & 1,634,005,986 & 88.89 & - & 1,635,617,503 \\
          \hline
    \end{tabular}
  \end{adjustbox}
  \label{tab:milp_scalability}
\end{table}

The average computation times for the small-scale scenarios were $1,450.656$~s, $1.525$~s, and $0.017$~s, for the Complete, Incomplete and Heuristic approaches, respectively. This demonstrates the lack of scalability for optimally solving our complete \ac{MILP} formulation, due to the large number of variables and constraints. Table~\ref{tab:milp_scalability} shows the number of variable and constraints, categorized by type, for a range of scenarios with increasing numbers of robots and tasks. The size of the \ac{MILP} instances grows exponentially, with a consistently significant proportion of integer variables, which are known to impact scalability. Among the problem's inherent variables and constraints, those related to time coordination are particularly numerous. This aligns with the understanding that the primary complexity stems from task fragmentation and multi-robot coordination. However, the overhead associated with linearization also contributes substantially to the problem's size. In summary, the results confirm that the MILP formulation's intractability is due to the large size of the resulting instances. Contrary to ours, sequence-based formulations use binary variables to indicate whether a task $i$ follows a task $j$ for a given shared resource, typically reducing the number of variables. However, the inherent complexity of our problem would still hinder scalability. For this reason, we opted not to pursue these alternative formulations, which are also less intuitive in our case.
Our significant overhead from linearization suggests that exploring Non-Linear Programming solvers could be a promising line of work. Nonetheless, this is beyond the scope of our current work. Given the problem's complexity, scalability challenges would likely persist even with more efficient exact solvers, reinforcing the need for heuristic solutions.

Next, we test the scalability of our heuristic planner further. For this, we generated a set of random scenarios as described in Section~\ref{sec:experimental_setup}, gradually increasing the size of the problem in terms of number of tasks and \acp{UAV}, running 100 scenarios per problem size.
We then solved each scenario with several variants of the heuristic planner:
1) \emph{Heuristic}: our heuristic planner as described in Section~\ref{sec:planner};
2) \emph{Random}: a version of our heuristic planner where the order in which tasks are allocated and the selection of robots are decided randomly;
3) \emph{Pseudo-Random}: a pseudo-random version where only the task allocation order is computed randomly, while robot selection is carried out using Algorithm~\ref{alg:RobotSelection};
and 4) \emph{Greedy}: a simpler greedy heuristic algorithm used as a baseline. It is inspired by the way that food buffets work. A predefined fixed order is used to build two queues: one with tasks (according to their execution time in descending order) and one with robots (in numerical order). Then the robots iteratively take as many tasks as they can handle before going to recharge and returning to the end of the queue, until all tasks are covered. 
Note that the four variants take care of coordination of the robot coalitions and of complying with all problem constraints, ensuring that all final solutions are valid plans. 

\begin{figure}[tb!]
  \centering
  
    \includegraphics[width=\columnwidth]{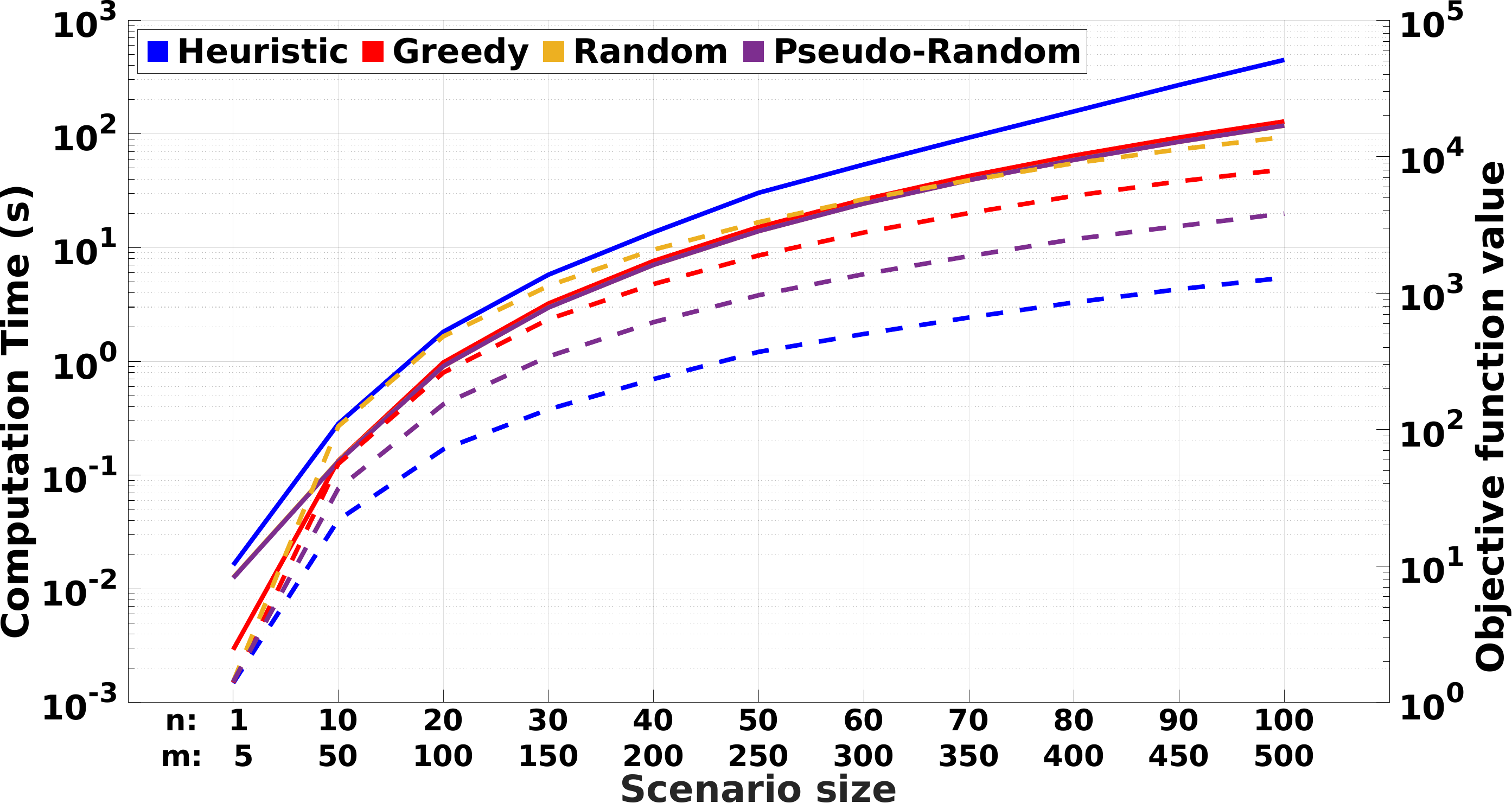}
    
  \caption{Evolution of computation time (solid lines) and the objective function value $f$ (dashed lines) for heuristic variants tested as the size of the scenarios (number of robots/tasks) increases. Average values for 100 scenarios are shown. Both metrics are shown on a logarithmic scale.}
  \label{fig:scalability_tests}
\end{figure}

Figure~\ref{fig:scalability_tests} shows the results of the scalability tests, both in terms of computation time and quality of the solutions.  
The quality of solutions is measured with the value of the objective function $f$; the lower, the better. From this perspective, our heuristic algorithm outperforms the others, with the improvement over the other variants increasing exponentially as the size of the scenario grows.
The Random version gives the worst results, followed by the Greedy version, which demonstrates the importance of applying a more \emph{intelligent} heuristic for such a highly constrained problem. 
Moreover, although the Pseudo-Random version performs better than Random and Greedy, it still produces significantly worse solutions than our Heuristic approach, which shows the importance of task allocation order. 
In terms of computation time, our Heuristic approach scales worse than the other variants, which show similar performance. Recall that the Heuristic variant recomputes the best robot coalition for each task at each iteration, and this robot selection also implies iterating over several coalitions until the best one is chosen. Therefore, the increase in computation time with the size of the scenario was expected. The other variants tested avoid some of this computation, as they simply pick tasks or robots randomly (Random and Pseudo-Random) or according to a predefined order (Greedy).  
Overall, although the computation time increase is exponential for our heuristic solver, the results show that we can handle very large scenarios, with a computational load that is reasonable for real-time planning (on the order of minutes in the worst case).  

\subsection{Plan repair and replanning}
\label{sec:replanning_exp}

In this section, we evaluate our whole mission replanning and execution framework in dynamic scenarios through a realistic simulation setup. The objective is twofold: we assess the performance of our algorithms for plan repair and replanning under robot delays and failures, and we show the feasibility of our approach for real applications.  

\begin{table}[tbh!]
	\centering
  \caption{Plan repair and replanning performance. The relative increments in the metrics are averaged over the successful scenarios for each column.}
	\begin{adjustbox}{max width=\columnwidth}
		\begin{tabular}{c|c c|c c|c c}
		  	\hline
		  	\multicolumn{1}{c}{Approach} & \multicolumn{2}{c}{Repair} & \multicolumn{2}{c}{Replanning} & \multicolumn{2}{c}{Combined} \\
			\hline
			Delay length     & Short           & Long             & Short            & Long              & Short            & Long \\
			\hline \hline
			SR (\%)          & $89.5$          & $23$             & $100$            & $45.5$            & $100$            & $45.5$ \\
			$\Delta f$ (\%)  & $3.47 \pm 0.46$ & $30.81 \pm 5.02$ & $65.38 \pm 6.67$ & $83.68 \pm 8.79$  & $11.73 \pm 3.61$ & $53.15 \pm 5.70$ \\
			$\Delta Z$ (\%)  & $0.20 \pm 0.02$ & $1.77  \pm 0.30$ & $1.74  \pm 0.28$ & $4.28  \pm 0.60$  & $0.60  \pm 0.16$ & $3.65  \pm 0.52$ \\
			$\Delta$WTR (\%) & $2.11 \pm 0.39$ & $22.53 \pm 4.40$ & $48.15 \pm 5.10$ & $56.50 \pm 6.41$  & $7.44  \pm 2.07$ & $37.51 \pm 4.59$ \\
			$\Delta$CBT (\%) & $0.26 \pm 0.02$ & $0.31  \pm 0.15$ & $0.23  \pm 0.02$ & $0.27  \pm 0.11$  & $0.25  \pm 0.02$ & $0.26  \pm 0.10$ \\
			$\Delta$WD (\%)  & $0.02 \pm 0.01$ & $0.29  \pm 0.19$ & $5.78  \pm 0.72$ & $5.30  \pm 1.09$  & $0.86  \pm 0.36$ & $2.20  \pm 0.71$ \\
			\hline
		\end{tabular}
	\end{adjustbox}
  \label{tab:plan_repair}
\end{table}

In a first experiment, we assessed the effectiveness of our plan repair and replanning methods using random trials involving unexpected robot delays. We generated 200 random scenarios (as described in Section~\ref{sec:experimental_setup}), each containing 50 tasks and 10 \acp{UAV}. For each scenario, we computed an initial plan using Algorithm~\ref{alg:HeuristicPlanner} and then randomly selected a~\ac{UAV} from the plan and a task from its schedule to apply a delay. We generated 400 trials in total: the first 200 trials involved sampling short delays uniformly distributed between 30 seconds, 1 minute, and 2 minutes for the previously computed plans; the other 200 trials involved sampling long delays uniformly distributed between 10, 15, and 20 minutes. Long delays were only applied to recharge tasks, as such delays would otherwise render the UAV inoperative due to battery depletion. 

Table~\ref{tab:plan_repair} compares three approaches for handling delayed plans: repairing them with Algorithm~\ref{alg:repairPlans}, replanning from the current state with Algorithm~\ref{alg:HeuristicPlanner}, or our approach, replanning in case repairing fails.~\footnote{Video with illustrative examples at: \url{https://youtu.be/a6zXwTZKFR4}} The repair approach exhibits a high success rate, which decreases with longer delays, as expected. Replanning may fail when the delayed robots are depleted of battery and insufficient robots remain to complete the mission. 
For the repair approach, short delays minimally affect performance metrics due to the algorithm's ability to distribute the delay across existing waiting times. The impact is more pronounced for longer delays, primarily affecting the WTR (and consequently, $f$), as these delays exceed the initial waiting times and necessitate additional waiting periods. CBT remains unaffected because waiting times are strategically placed during recharges whenever possible. WD is also minimally impacted due to inter-robot synchronization; when one queue is delayed, others adjust proportionally, maintaining balanced WD. For replanning, results are similar for both short and long delays, as this approach does not explicitly accommodate existing delays. WTR (and consequently, $f$) is significantly worse after replanning. The heuristic planner disregards the initial plan and generates a new one from the delayed state, often leading to substantially different solutions. These results demonstrate that combining repair with replanning, rather than simply recomputing a new plan, leads to a higher SR and more effective handling of robot delays in terms of performance. Moreover, in real-world operations, repairing a plan offers practical benefits by avoiding task reassignment: once a multi-\ac{UAV} flight plan is established, operators typically prefer adjusting waiting times, as reassignment requires new flight plans and complicates safety checks.

\begin{figure}[tb]
  \centering
  \includegraphics[width=0.65\columnwidth]{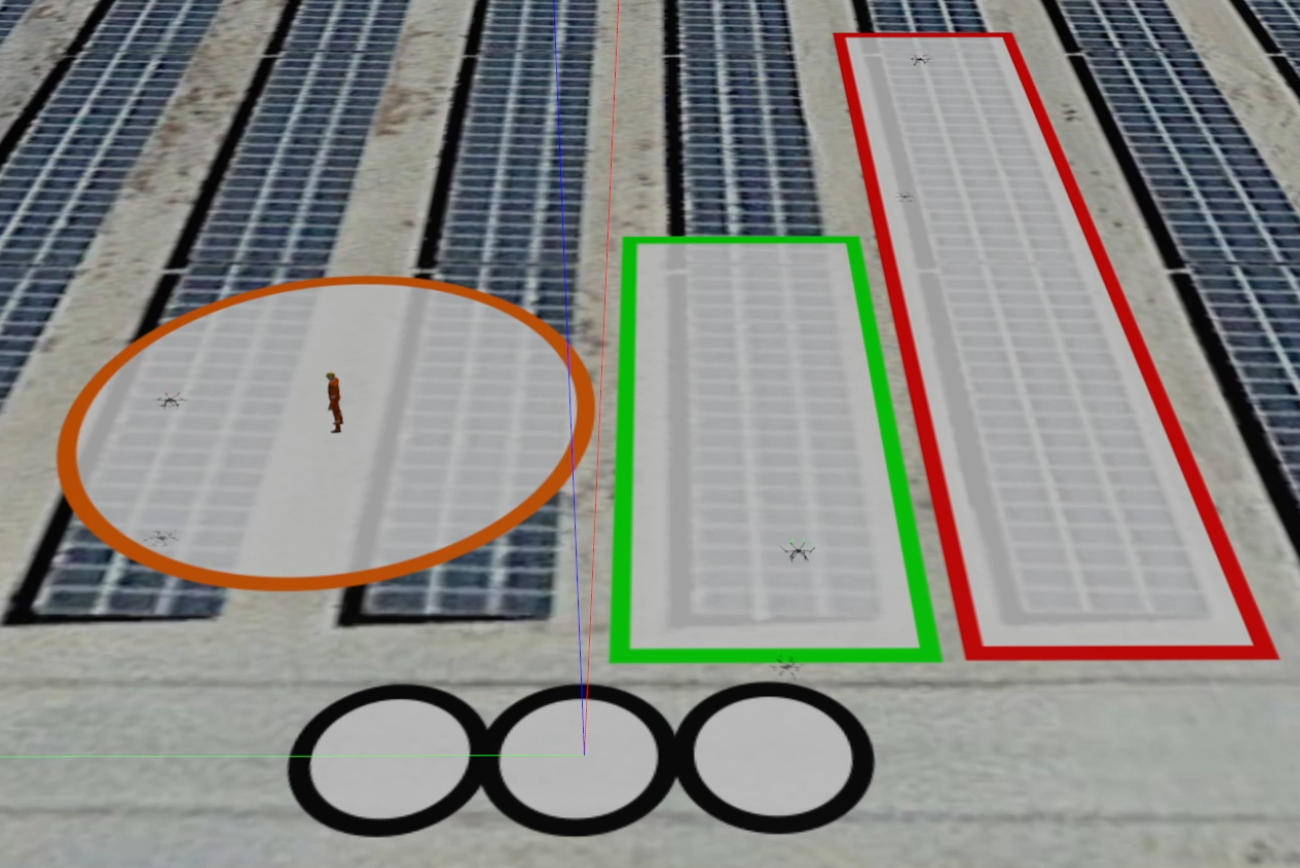}
  \caption{Simulation of the solar plant use case. Orange, green and red areas represent monitoring and inspection zones.}
  \label{fig:snapshot}
\end{figure}

Finally, we ran an experiment to demonstrate the replanning feature through an illustrative example. For that, we implemented our system architecture for mission replanning and execution (Section~\ref{sec:online_replaning}) integrated with ROS (Robot Operating System), and we created a Gazebo simulation of the Evora solar plant for our use case (Section~\ref{sec:use_case}).\footnote{Code at~\url{https://github.com/multirobot-use/mrta_execution_architecture}.} The~\acp{UAV} were modeled using a software-in-the-loop tool to simulate the PX4 autopilot firmware. This setup allowed us to reproduce real application scenarios quite closely, in which the system could not distinguish between simulated or actual \acp{UAV}. Our MATLAB code for mission planning was integrated into the ROS architecture using the ROS Toolbox for MATLAB.\footnote{https://www.mathworks.com/products/ros.html} Using ROS Actionlib, we created an \emph{Action Server} in MATLAB that receives planning requests from a High-Level Planner module in ROS and communicates the resulting plans to the \acp{UAV}. In the experiment, a team with 3 available \acp{UAV} was assigned a mission with 3 tasks: 2 \emph{inspection} tasks, and 1 \emph{monitoring} task. While executing the initial plan, computed by the heuristic planner, where each robot was assigned one task, we simulated a failure in the \ac{UAV} performing the \emph{monitoring} task, which was then reassigned to one of the other \acp{UAV}. Figure~\ref{fig:snapshot} shows a screenshot of the simulation.\footnote{Full video available at: \url{https://youtu.be/2hzP7LZRd0g}}

\section{Conclusions}
\label{sec:conclusions}

In this paper, we have presented a planning framework for heterogeneous \ac{MRTA} in long-endurance missions for dynamic scenarios. We have formulated an optimization problem as a \ac{MILP} that integrates robot recharges, heterogeneous robot capabilities, task fragmentation/relays, and time coordination for multi-robot tasks. Our results demonstrate that the aggregation of such a diverse set of features can help us solve complex missions that are relevant to a wide spectrum of multi-robot applications that would otherwise be unsolvable. To achieve better scalability for large scenarios and real-time planning performance, we have also proposed a heuristic solver for our \ac{MRTA} problem and integrated it into a mission planning and execution architecture capable of repairing or recomputing plans online as unexpected circumstances arise. Our experimental results in a realistic use case for multi-\ac{UAV} inspection demonstrate that 1) the flexibility introduced by our decomposable and varying coalition size tasks enables us to improve the defined performance metrics for the resulting plans; 2) our heuristic solver can scale for large scenarios in terms of computation time and outperform other similar variants in terms of efficacy; and 3) our replanning method can repair plans for unexpected robot delays.

It is important to highlight that the complexity of our~\ac{MRTA} problem makes it difficult to find an alternative planner for comparison in the state of the art. While we have discussed some related works in Section~\ref{sec:relatedWork}, we did not find scalable solvers in the literature that can tackle problems integrating all our constraints simultaneously. Metaheuristic methods that can solve generic optimization problems, such as genetic algorithms or simulated annealing, could be an option, but the standard implementations of these methods do not exploit inherent problem properties to achieve better performance and do not scale when the number of variables in the formulation increases significantly. For example, we have tried to solve the scenarios in Section~\ref{sec:scalability_exp} with the genetic algorithm implemented in the MATLAB Optimization Toolbox without success. An interesting direction for future work could be the design of a tailored metaheuristic algorithm that leverages particular features of the problem to improve efficiency in the search for possible solutions, and then use it for comparison. The exploration of more efficient \ac{MILP} formulations is another promising avenue for future work. Finally, we plan to conduct field experiments with a real team of UAVs performing inspection missions to demonstrate our planning framework, evaluating communication latency and overhead.

\balance




\bibliographystyle{IEEEtran}
\bibliography{IEEEabrv, bib_short}

\end{document}